\documentclass[10pt,twocolumn,letterpaper]{article}

\usepackage{cvpr}              

\usepackage{graphicx}
\usepackage{amsmath}
\usepackage{amssymb}
\usepackage{booktabs}

\usepackage{gensymb}
\usepackage[pagebackref,breaklinks,colorlinks]{hyperref}

\def\cvprPaperID{4583} 
\def\confName{CVPR}
\def\confYear{2022}

\usepackage[utf8]{inputenc} 
\usepackage[T1]{fontenc}    
\usepackage{hyperref}       

\usepackage{booktabs}       
\usepackage{amsfonts}       
\usepackage{nicefrac}       
\usepackage{microtype}      
\usepackage{xcolor}         
\usepackage{amsmath,mathtools}
\usepackage{amssymb}
\usepackage{amsthm}
\usepackage{graphicx}
\usepackage{enumitem}
\usepackage{bm}
\usepackage{bbm}
\usepackage{enumitem}
\usepackage{tabularx}
\usepackage{multirow}
\newcommand{\one}{\mathbf{1}}
\newcommand{\zero}{\mathbf{0}}
\newcommand{\Eye}{\mathbf{I}}

\newcommand{\B}{\mathbf{B}}
\newcommand{\U}{\mathbf{U}}
\newcommand{\V}{\mathbf{V}}
\newcommand{\W}{\mathbf{W}}

\newcommand{\D}{\mathbf{D}}

\newcommand{\Hm}{\mathbf{H}}
\newcommand{\M}{\mathbf{M}}
\newcommand{\N}{\mathbf{N}}
\newcommand{\PC}{\mathbf{P}}

\newcommand{\R}{\mathbf{R}}
\newcommand{\Rbb}{\mathbb{R}}

\newcommand{\s}{\mathbf{s}}

\newcommand{\g}{\mathbf{g}}

\newcommand{\p}{\mathbf{p}}
\newcommand{\n}{\mathbf{n}}
\newcommand{\barrel}{\text{barr}}
\newcommand{\base}{\text{base}}

\newcommand{\eaxis}{\mathbf{e}}
\newcommand{\ecenter}{\mathbf{c}}
\newcommand{\sketch}{\mathbf{S}}
\newcommand{\profile}{\tilde{\mathbf{S}}}

\newcommand{\gc}{E}
\newcommand{\loss}{\mathcal{L}}
\newcommand{\weights}{\bm{\phi}}
\newcommand{\Weights}{\bm{\Phi}}
\newcommand{\network}{\mathcal{G}}
\newcommand{\pars}{\bm{\theta}}
\newcommand{\parsIGR}{\bm{\beta}}

\newtheorem{thm}{Theorem}

\newtheorem{dfn}{Definition}

\newcommand{\parallelsum}{\mathbin{\!/\mkern-5mu/\!}}

\DeclareMathOperator*{\argmax}{arg\,max}
\DeclareMathOperator*{\argmin}{arg\,min}

\newcommand{\methodname}{{{Point2Cyl}}\xspace}
\newcommand{\primname}{{{Extrusion Cylinder}}\xspace}
\newcommand{\primnames}{{{Extrusion Cylinders}}\xspace}

\usepackage[capitalize]{cleveref}
\crefname{section}{Sec.}{Secs.}
\Crefname{section}{Section}{Sections}
\Crefname{table}{Table}{Tables}
\crefname{table}{Tab.}{Tabs.}

\Crefname{assumption}{\textbf{H}\hspace{-3pt}}{\textbf{H}\hspace{-3pt}}
\crefname{algorithm}{\text{Alg.}}{\text{Alg.}}
\crefname{assumption}{\textbf{H}}{\textbf{H}}
\crefname{equation}{\text{Eq}}{\text{Eq}}
\crefname{definition}{\text{Dfn.}}{\text{Dfn.}}
\crefname{lemma}{\text{Lemma}}{\text{Lemma}}
\crefname{dfn}{\text{Dfn.}}{\text{Dfn.}}
\crefname{thm}{\text{Thm.}}{\text{Thm.}}
\crefname{tab}{\text{Tab.}}{\text{Tab.}}
\crefname{fig}{\text{Fig.}}{\text{Fig.}}
\crefname{table}{\text{Tab.}}{\text{Tab.}}
\crefname{figure}{\text{Fig.}}{\text{Fig.}}
\crefname{section}{\text{Sec.}}{\text{Sec.}}

\definecolor{mkcolor}{RGB}{255,0,128}

\definecolor{yycolor}{RGB}{0,128,128}

\definecolor{mhcolor}{RGB}{0,128,0}

\definecolor{tolgacolor}{RGB}{128, 0, 0}

\definecolor{joecolor}{RGB}{128,0,128}

\definecolor{leocolor}{RGB}{0,0,255}

\definecolor{purvicolor}{RGB}{128,0,255}

\definecolor{purvi_edit_color}{RGB}{255,0,128}

\renewcommand{\paragraph}[1]{{\vspace{1mm}\noindent \bf #1}.}

\makeatletter
\newenvironment{mcases}[1][l]
 {\let\@ifnextchar\new@ifnextchar
  \left\lbrace
  \array{@{}l@{\quad}#1@{}}}
 {\endarray\right.}
\makeatother

\newcommand{\myqed}{\rlap{$\quad\qquad \Box$}}

\newcommand{\insertimageC}[5]{ 
\begin{figure}[#5]
\centering
\includegraphics[width=#1\linewidth, clip=true]{figures/#2}
\vspace{-0.4cm}
\caption{#3}
\label{#4}
\end{figure}
}

\usepackage{wrapfig}
\usepackage{float}

\title{\methodname: Reverse Engineering 3D Objects\\from Point Clouds to \primnames \vspace{-0.2cm}\\}
\author{Mikaela Angelina Uy$^{*1}$~~~Yen-Yu Chang$^{*1}$~~~Minhyuk Sung$^2$~~~Purvi Goel$^1$
\vspace{0.1cm}\\
Joseph Lambourne$^3$~~~Tolga Birdal$^1$~~~Leonidas Guibas$^{1}$
\vspace{0.2cm}\\
$^1$Stanford University~~~$^2$KAIST~~~$^3$Autodesk Research
\vspace{0.3cm}\\
}
\date{June 2022}

\begin{document}

\twocolumn[{
\renewcommand\twocolumn[1][]{#1}%
\maketitle
\vspace{-0.485in}
\begin{center}
    \centering
    \includegraphics[width=\textwidth]{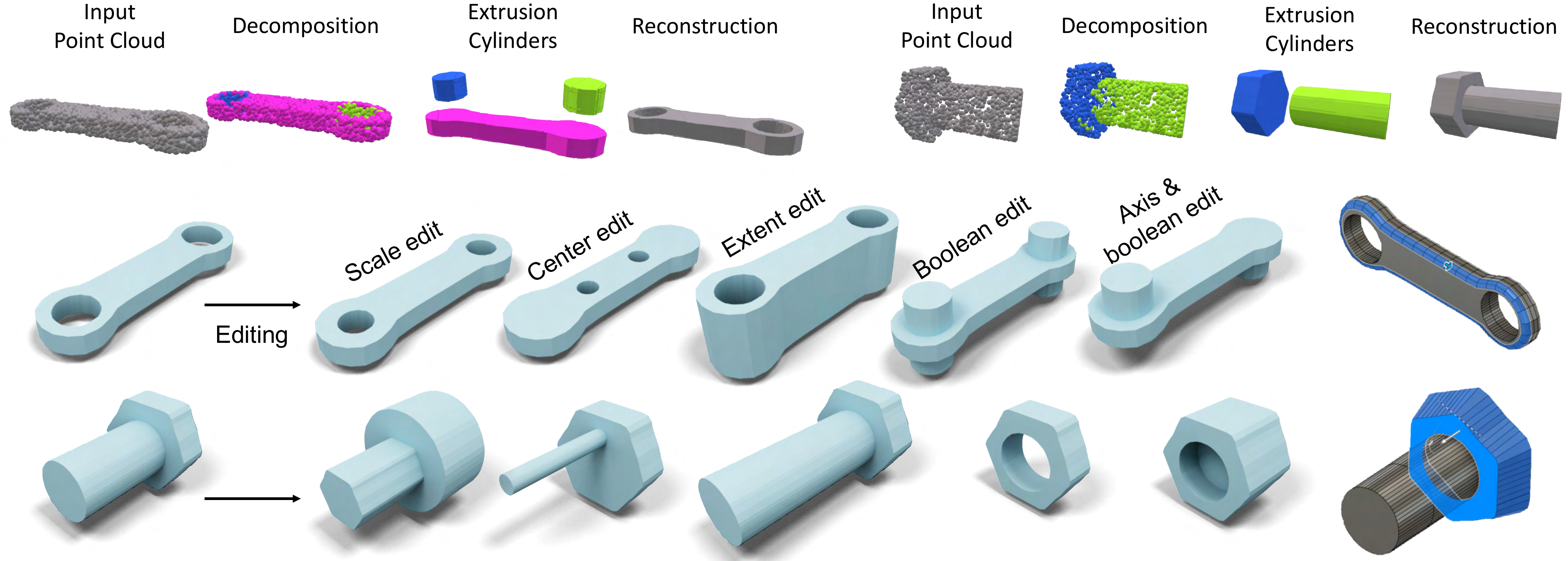}
    \captionof{figure}{\textbf{Overview}.  
    \textbf{\methodname} takes a raw point cloud as input and decomposes it into extrusion cylinders while predicting all parameters including the extrusion axis, extent, and the 2D sketch (first row). The output set of extrusion cylinders can be loaded in CAD software and is editable in various ways thus creating a wide array of variations (second and third rows).
    }
    \label{fig:teaser}
\end{center}%
}]
\maketitle

\begin{abstract}
\vspace{-4\baselineskip}
We propose \textbf{Point2Cyl}, a supervised network transforming a raw 3D \textbf{point} cloud \textbf{to} a set of extrusion \textbf{cylinders}. Reverse engineering from a raw geometry to a CAD model is an essential task to enable manipulation of the 3D data in shape editing software and thus expand their usages in many downstream applications. Particularly, the form of CAD models having a sequence of \emph{extrusion cylinders} --- a 2D sketch plus an extrusion axis and range --- and their boolean combinations is not only widely used in the CAD community/software but also has great expressivity of shapes, compared to having limited types of primitives (e.g., planes, spheres, and cylinders). 
In this work, we introduce a neural network that solves the extrusion cylinder decomposition problem in a geometry-grounded way by first learning underlying geometric proxies. Precisely, our approach first predicts per-point segmentation, base/barrel labels and normals, then estimates for the underlying extrusion parameters in differentiable and closed-form formulations. Our experiments show that our approach demonstrates the best performance on two recent CAD datasets, Fusion Gallery and DeepCAD, and we further showcase our approach on reverse engineering and editing. 
\end{abstract}
\vspace{-1cm}
\section{Introduction}\label{sec:intro}
\footnotetext{$*$ denotes equal contribution}
Our everyday environments are filled with objects fabricated following a carefully engineered computer-aided design.
This makes \textit{reverse engineering in the wild} a vital workflow in situations where copies or variations of a physical object are required, but the corresponding CAD model is not available~\cite{VARADY1997255}. This situation often occurs when repairing machinery or digitizing objects manufactured in the pre-digital era~\cite{reverse_engineering_modeling_methods}. 
To this end, an object is first scanned using a 3D sensor producing a point cloud and later decomposed into a set of consistent primitives or surfaces which could be parsed by existing shape modelers such as Fusion360~\cite{fusion360} or SolidWorks~\cite{solidworks2005solidworks}. 
However, at the user level, a CAD model is designed as a sequence of operations, where the designer first draws a planar 2D \emph{sketch} as a closed curve and later \emph{extrudes} it into a 3D solid~\cite{wu2021deepcad} (\cf~\cref{fig:sketchextrude}). These \emph{\primnames} can then be combined through \emph{boolean} operations.

As this modeling paradigm is hard to summarize using traditional primitives such as planes or cylinders, we set off to ask: \emph{how can we reverse engineer point clouds into primitives that are interpretable and usable in the modeling process of CAD designers?}
Traditional approaches answer this question by following a three-step procedure where (i) the point cloud is first converted into a mesh, (ii) subsequently explained by a collection of trimmed parametric surfaces resulting in a watertight solid (a.k.a.~boundary representation or B-Rep)~\cite{segmentation_methods_for_smooth_point_regions},
and (iii) a CAD program, which could generate the input B-Rep is inferred~\cite{zonegraph}. 
Recent trends in fitting primitives to point clouds~\cite{spfn,birdal2019generic,sommer2020primitect} can bypass the initial meshing stage, but they either assume a finite set of fixed primitives, \eg, planes, cylinders,
cones~\cite{spfn,drost2015local,tian2019learning} or output a disjoint set of primitives yet to be stitched~\cite{birdal2018minimalist,superquadrics} and thus cannot allow, for example, convenient \emph{shape editing} or \emph{variations}. Note that geometric primitives have unique parameterizations and hence cannot be handled by general 3D model detection pipelines like~\cite{qi2019deep,misra2021-3detr}. As such, the presented problem is much finer-grained than explaining a scene with a retrieved set of CAD models, as done in~\cite{scan2cad,uy2020deformation}. Finally, both of the problems we address, \emph{CAD model reconstruction from point sets} and \emph{user editing of CAD shapes} are cast as future work by the recent CAD generative model, DeepCAD~\cite{wu2021deepcad}.

To be able to achieve a geometry-grounded and editable reconstruction, within a CAD-grammar, we propose to cast the problem as a decomposition task into \emph{\primname}s. Our novel approach views \emph{\primname} as a parametric \emph{primitive} can jointly represent a set of sketch-extrude operations and hence is suitable for representing CAD models.
We learn to decompose a given raw point cloud into \primname instances. In particular, our neural network learns to predict per-point \emph{extrusion instance segmentation}, \emph{surface normal}, and \emph{base/barrel membership}\footnote{Here, barrel means the cylinder side, the surface swept by the sketch.}. 
Given this decomposition, we show how to solve for each primitive's underlying parameters, including the extrusion \emph{axis}, \emph{center}, \emph{sketch}, sketch \emph{scale}, and extrusion \emph{extent}, all in a geometry-grounded way. 
In return, as shown in~\cref{fig:teaser}, these enable us to reverse engineer the point cloud into an editable 3D CAD model in a format that is directly consumable by existing CAD modelers~\cite{fusion360,opencascade}, allowing for further variation creation, \eg, adding fillets/chamfers, modifying the sketch and varying parameters such as center and extents.



In summary, our contributions are the following:
\begin{enumerate}[leftmargin=\parindent,noitemsep,topsep=0.2em]
    \item We introduce a novel approach that casts the 3D reconstruction task as a \primname decomposition problem, making it well-suited for CAD modeling.
    \item We architect a neural network that decompose an input point cloud into a set of Extrusion Cylinders by learning geometric proxies, which can then used to estimate the extrusion parameters in differentiable, closed-form formulations.
\item We validate our approach quantitatively and qualitatively on two existing CAD datasets, Fusion Gallery~\cite{willis2020fusion} and DeepCAD~\cite{wu2021deepcad}, surpassing baselines, and further showcase its applications on reconstruction and shape editing. 
\end{enumerate}




\noindent Our project page can be found at \href{https://point2cyl.github.io}{point2cyl.github.io}.

\insertimageC{1}{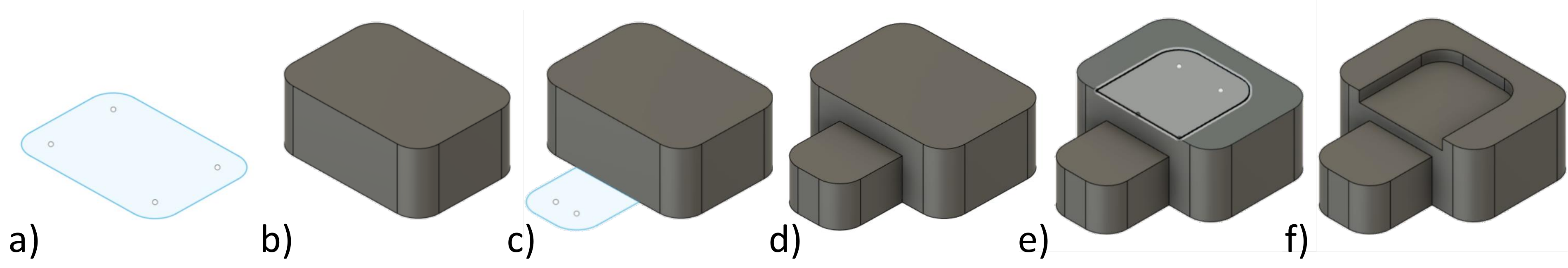}{Solid model creation as a sequence of sketch-extrude operations. \textbf{a}) Initial sketch. \textbf{b}) Volume extruded from that sketch. \textbf{c}) A second sketch. \textbf{d}) The second sketch is extruded and the boolean union of the two extruded volumes is created. \textbf{e}) A third sketch. \textbf{f}) The final sketch extruded and subtracted from the solid. \vspace{-3mm}}{fig:sketchextrude}{t!}
\section{Related work}


\paragraph{Primitive fitting to point clouds}
In vision literature, the primitive fitting and object decomposition have been investigated for decades with diverse types of primitives.
The simplest forms of primitives are planes, which have attracted significant attention as they are omnipresent in our environments~\cite{borrmann20113d,monszpart2015rapter,sommer2020planes,fang2018planar,czerniawski20186d}.
More general types of primitives were also explored in the decomposition with RANSAC~\cite{Schnabel2007,Tran2015,li2011globfit} and region-growing~\cite{oesau2016planar} approaches.
The local Hough transform of Drost and Ilic \cite{drost2015local} showed how the detection of primitives can be made more efficient by considering the local voting spaces. This idea is quickly extended to handle a wider range of primitives~\cite{sommer2020primitect,birdal2019generic}. Conic sections are also of special interest as they allow for an infinite set of variations~\cite{Andrews2014,birdal2018minimalist,morwald2013geometric,frahm2006ransac}. 

The proliferation of deep learning has steered researchers to data-driven frameworks~\cite{cpfn}. Li \etal~\cite{spfn} is the first to use a supervised learning method for the primitive fitting, and Paschalidou~\etal~\cite{paschalidou2019superquadrics} and Sharma~\etal~\cite{sharma2020parsenet} extended the idea to fit superquadrics and B-spline surfaces to a point cloud, respectively. However, all these methods assume a fixed set of primitives and none of them handle our \primnames in their fittings.


\paragraph{Reverse engineering in CAD world}
In CAD and graphics communities,
reverse engineering is a well studied problem with many traditional algorithms explored over the years \cite{BENIERE20131382, VARADY1997255}.  Benko \etal  \cite{segmentation_methods_for_smooth_point_regions} described a commonly used procedure for segmenting triangle meshes by analysing the pattern which the normal vectors form on the Gaussian sphere.  This allows planes, cylinders, cones and doubly curves surfaces to be identified. Following segmentation, these primitives can be fitted 
using geometric constraints \cite{Benk2002ConstrainedFI}. 

It is only with the recent availability of large B-Rep, sketch and construction sequence datasets \cite{koch2019abc, Ari2020, willis2020fusion, wu2021deepcad} that the machine learning community has become interested in the generation of CAD data  \cite{sharma2020parsenet, wang2020pienet, Li:2020:Sketch2CAD, Ari2020, willis2021engineering, willis2020fusion, Ganin2021ComputerAidedDA, para2021sketchgen, zonegraph, wu2021deepcad, seff2021vitruvion}. 
Wang \etal \cite{wang2020pienet} showed results for identifying keypoints in a point cloud suitable for fitting with parametric curves. Sharma \etal \cite{sharma2020parsenet} decomposed a point cloud into patches suitable for fitting with parametric surfaces. In both cases the difficult task of combining individual curves or surface to build watertight solid models which can be worked with in CAD modelers is not addressed.        
Ganin \etal \cite{Ganin2021ComputerAidedDA} use a transformer model to generate 2D sketches. While their model can be conditioned on 2D image data, this data must be a sequence raster images of the individual curves which build up the model -- a sequential breakdown not available in the reverse engineering setting. 
Tian \etal \cite{tian2019learning} showed how shapes can be interpreted as programs represented using an RNN and voxel based "neural renderer".



Most relevant to us is DeepCAD~\cite{wu2021deepcad}, which unconditionally generated entire 3D CAD models using a transformer based method. While DeepCAD demonstrates excellent results when auto-encoding in the ``program space", it leaves the guidance of the generated shape using point clouds for  ``future applications", which is what we address in this work.
\section{The Extrusion Cylinder}
\label{sec:gencylinder}

Before delving into the methodological specifics, we first define the \primname, a \emph{primitive} that gives us the flexibility of creating any shape from arbitrary closed loops, by composing them through a series of boolean operations, mirroring the CAD design process. The extrusion cylinder, represent solid volumes in contrast to existing surface-based works~\cite{spfn, parsenet}, and is parametric as defined below. Our definitions are illustrated in~\cref{fig:gencylinder}. We further present closed-form and  differentiable formulations to recover the extrusion parameters from points.

\begin{dfn}[Sketch and profile]\label{dfn:sketch}
We consider a non-self intersecting, finite area, closed loop and normalized 2D \textbf{sketch} $\profile=\{p(q(t)) \in \Rbb^2|t\in[0,1], p(q(0))=p(q(1))\}$, for continuous functions $q:[0,1]\rightarrow \Rbb$ and $p:\Rbb\rightarrow \Rbb^2$. The area enclosed by $\profile$ is often called a \textbf{profile}.
\end{dfn}
\begin{dfn}[Sketch plane]
We also define the plane containing $\profile$, parameterized by the center $\ecenter\in\Rbb^3$, and a normal along the axis, $\eaxis\in\mathbb{S}^2$. 
\end{dfn}
Note that the sketch $\profile$ defines a profile on the sketch plane parameterized by $(\ecenter, \eaxis)$ without ambiguity.

\begin{dfn}[Extrusion Cylinder]
In manufacturing, \textbf{extrusion} is the process of pushing the material forward along a fixed cross-sectional profile to a desired height. In our context, we use extrusions to parameterize our primitive, the \text{extrusion cylinder}, by an axis $\eaxis \in \mathbb{S}^2$, a center $\ecenter \in \Rbb^3$ associated to a sketch $\profile$ scaled by $s\in\Rbb$\footnote{Scale $s$ allows normalized sketch $\profile$ to always fit in a unit circle.}. We further introduce the extents $(r^\text{min}, r^\text{max}) \in (\Rbb\times\Rbb)$ defining the extrusion $\gc=(\eaxis,\ecenter,\profile,s,r^\text{min}, r^\text{max})$.
\end{dfn}

Compared to existing primitive fitting~\cite{spfn,parsenet} or CSG~\cite{csgnet,ucsgnet, ren2021csg} works, we do not assume a finite set of primitives, and instead consider the building blocks to be an arbitrary sketch-extrude operation applied on any closed, non-self-intersecting 2D loop. As such, our extrusion cylinder constitutes a typical building block for CAD pipelines as those used in solid modellers. As we work with point clouds inputs, we now define an attribute associated with points along the \emph{surface} or the \emph{boundary} of the extrusion cylinder.

\begin{dfn}[Base \& Barrel]\label{dfn:bb}
We classify points along the surface of an extrusion cylinder as being base or barrel points. Base points are points that lie on the plane at either extents of the extrusion cylinder, while barrel points are points that lie along the "sides" of the extrusion cylinder. Hence, the surface normals of base/barrel points are parallel/perpendicular to the extrusion axis $\eaxis$.

For any point incident to the boundary of an extrusion cylinder $(\p_i\in\Rbb^3)\in \gc$, we represent this attribute by $b_i$. $b_i=0$ for the barrel and $b_i=1$ for bases. With these, it is easy to verify that for an extrusion axis $\eaxis\in\mathbb{S}^2$:
\begin{equation}\label{eq:basebarrelcond}
    \begin{mcases}[cc@{\ }l]
b_i=0 \quad : & \n_i \perp \eaxis & \mapsto & \quad \n_i^\top\eaxis=0 \\
b_i=1 \quad : & \n_i \parallelsum \eaxis & \mapsto & \quad \n_i^\top\eaxis=\pm 1.
\end{mcases}
\end{equation}
where $\n_i\in\mathbb{S}^2$ is the surface normal evaluated at point $\p_i$.
\end{dfn}

\begin{figure}
    \centering
    \includegraphics[width=\linewidth]{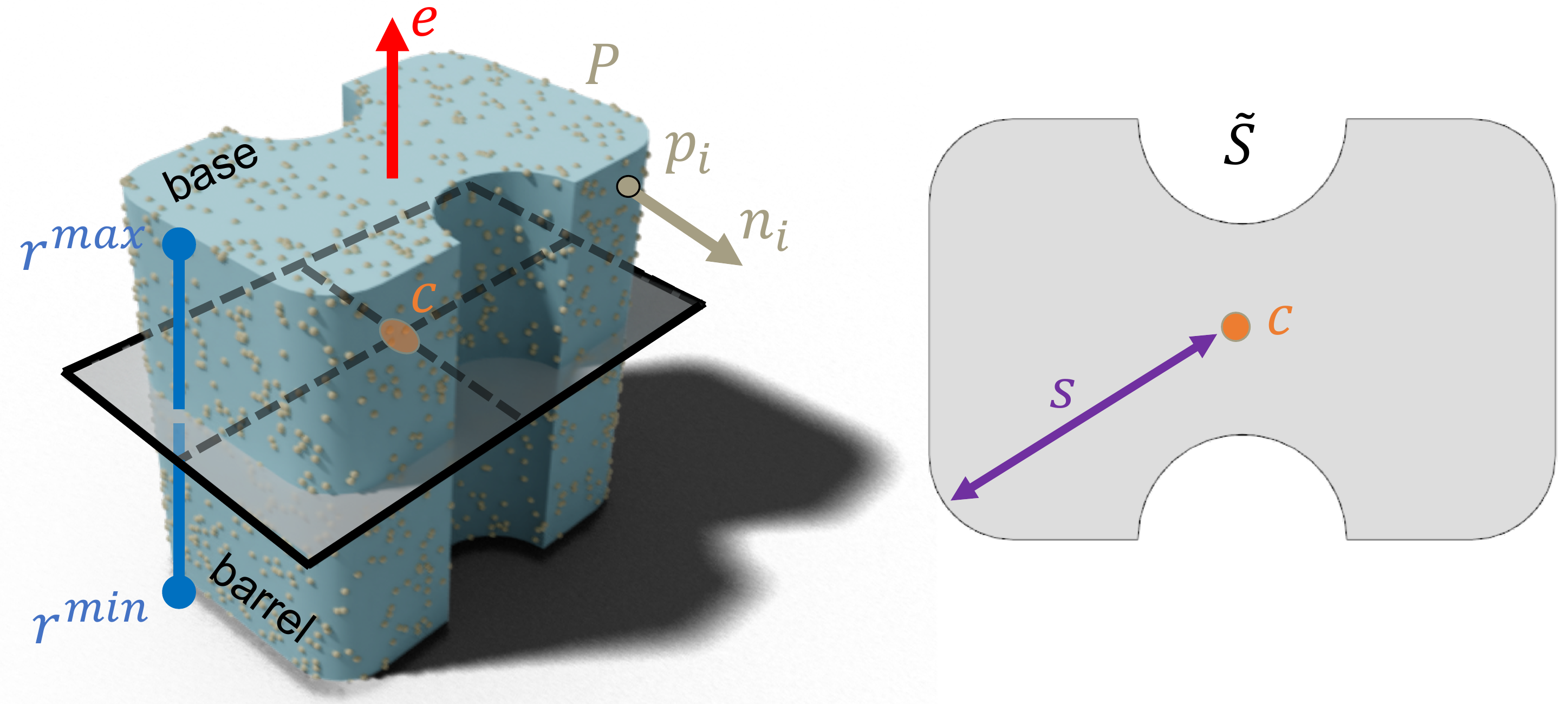}
    \vspace{-0.7cm}
    \caption{An illustration of our \textbf{extrusion cylinder.}}
    \vspace{-0.5cm}
    \label{fig:gencylinder}
\end{figure}

\paragraph{Recovering extrusion cylinders from points.} We now discuss how to recover the parameters of an extrusion cylinder $\gc$ from a set of points $\PC=\{\p_i\in\Rbb^3\}_{i=1}^N$ and corresponding normals $\N=\{\n_i\in\mathbb{S}^2\}_{i=1}^N$ incident to $\gc$. We let $\mathbf{P}_\base, \mathbf{P}_\barrel \subset\mathbf{P}$ denote base and barrel points of $\mathbf{P}$, respectively, where $\mathbf{P}=\mathbf{P}_\base \cup \mathbf{P}_\barrel$. The \textbf{center of the extrusion} ($\hat{\ecenter}$) is the simplest and can be estimated by the taking the mean of all the barrel points of $\PC$. 

The rest of the parameters depend upon the extrusion axis, for which we give the following algorithm:
\begin{thm}[Recovering \textbf{extrusion axis} from points]
For a set of points on an extrusion cylinder, the optimal extrusion axis is given by~\footnote{Throughout the paper, we denote $\widehat{\Box}$ and $\Box$ to be the predicted/estimate and ground truth values, respectively.}:
\begin{equation}\label{eq:thm1}
\hat{\eaxis}=\argmin_{\eaxis, ||\eaxis||=1}\,\eaxis^\top\Hm\eaxis,
\end{equation}
where $\Hm=(\mathbf{N}_\barrel^\top\mathbf{N}_\barrel-\mathbf{N}_\base^\top\mathbf{N}_\base)$, while $\mathbf{N}_\base\in \Rbb^{N_\base\times 3}\subset\N$ and $\mathbf{N}_\barrel\in \Rbb^{N_\barrel\times 3}\subset\N$ denote the corresponding normals for the base/barrel points belonging to the extrusion cylinder. The solution can be obtained by the eigenvector corresponding to the smallest eigenvalue of $\Hm$.
\end{thm}
\begin{proof}\renewcommand{\qedsymbol}{}
For an oriented point cloud $(\PC, \N)$, we could write the relations in~\cref{eq:basebarrelcond} as $\mathbf{N}_\barrel \eaxis =\zero$ and $\mathbf{N}_\base \eaxis =\pm\one$, where matrices $\mathbf{N}_\base$ and $\mathbf{N}_\barrel$ correspond to the point normals belonging to base ($\mathbf{P}_\base\subset\PC$) and barrel ($\mathbf{P}_{\barrel}\subset\PC$) points, respectively.
Given that the scalar product for normalized vectors is bounded $[0,1]$, for $\eaxis$ to be a proper extrusion direction, $\|\N_\base \eaxis\|\rightarrow \max$ and $\|\N_\barrel \eaxis\|\rightarrow \min$ must hold, hence:
\begin{align*}
\hat{\eaxis}=
\begin{cases}
\argmax_{\eaxis, \|\eaxis\|=1}\,\eaxis^\top\N_\base^\top\N_\base \eaxis \,\, \text{ for }\,  \N_\base \eaxis=\pm\one\\
\argmin_{\eaxis, \|\eaxis\|=1}\,\eaxis^\top\N_\barrel^\top\N_\barrel \eaxis \,\,\,\, \text{ for }\, \N_\barrel \eaxis=\zero.  
\end{cases}
\end{align*}
Negating the first and combining both objectives, we get: 
\begin{equation*}
\hat{\eaxis}=\argmin_{\eaxis, \|\eaxis\|=1}\,\eaxis^\top(\N_\barrel^\top\N_\barrel-\N_\base^\top\N_\base)\eaxis. \myqed
\end{equation*}\vspace{-7mm}
\end{proof}
For cases, where the data is contaminated with outlier points non-incident to the extrusion cylinder, we can further introduce a per-point weight factor $\weights\in\Rbb^{N}$, controlling the contribution of each of the normals in $\N$.
\begin{thm}[Weighted recovery of \textbf{extrusion axis} from points]\label{thm:weightedEAxis}
For a general weighted point set, the optimal \textbf{extrusion axis} is given by $\hat{\eaxis}=\argmin_{\eaxis, ||\eaxis||=1}(\eaxis^\top \Hm_{\Weights} \eaxis)$, where:
\begin{equation}\label{eq:eigen}
\Hm_{\Weights} = \N^\top\Weights_\barrel^\top\Weights_\barrel\N-\N^\top\Weights_\base^\top\Weights_\base\N.
\end{equation}
where $\Weights_\barrel = \emph{diag}(\weights_\barrel), \Weights_\base= \emph{diag}(\weights_\base) \in\R^{N\times N}$. $\weights_\barrel$/$\weights_\base$ indicate the barrel/base weights assigned to all points, respectively. 
The solution is again given by the eigenvector corresponding to the smallest eigenvalue of $\Hm_{\Weights}$.
\end{thm}
\begin{proof}[Proof sketch]\renewcommand{\qedsymbol}{}
$\Hm_{\Weights}$ is a direct modification of $\Hm$ incorporating the weights. Full proof is given in our supplementary.
\end{proof}
\vspace{-0.2cm}
The remainder of the parameters can be estimated as follows. We first define an operator $\Pi(\cdot)$: $\mathbf{q}_i = \Pi(\mathbf{p}_i, \bar{\eaxis}, \bar{\ecenter})$,  takes a point $(\mathbf{p}_i \in \mathbb{R}^3)$ to $(\mathbf{q}_i \in \mathbb{R}^2)$, by aligning $\bar{\eaxis}$ to the z-axis, projecting $\mathbf{p}_i , \ecenter$ onto the $xy-$plane, then centering $\bar{\ecenter}$ at the origin. We define the operator $\mathbf{\Pi}(\cdot)$ that does the same mapping for a 3D point cloud $\mathbf{P}$. Hence, an \emph{unnormalized} sketch can be approximated by a 2D point cloud $\mathbf{\Pi}(\mathbf{P}_\barrel, \hat{\eaxis}, \hat{\ecenter})$, and from here, the \textbf{extrusion scale} $\hat{s}$ can be computed by taking the distance from the farthest point to the origin. This then gives us the \emph{normalized} sketch approximated by $\bar{\sketch} = \hat{s}_i\mathbf{\Pi}(\mathbf{P}_\barrel, \hat{\eaxis}, \hat{\ecenter})$.

The \textbf{extrusion extent} is estimated similarly by calculating the minimum and maximum range of $\mathbf{P}_\barrel$ along $\eaxis$ as illustrated in Fig~\ref{fig:gencylinder}.\footnote{Only barrel points are used to calculate extent as some extrusion cylinders have no base points.}
Altogether: 
\begin{gather}
    \hat{s} = \max_{\p_i\in \PC_\barrel} ||\Pi(\p_i,\hat{\eaxis},\hat{\ecenter}))||_2 \label{eq:gcest}\\
    \hat{r}^{\text{min}} = \min_{\p_i\in \PC_\barrel} (\hat{\eaxis}\cdot(\p_i-\hat{\ecenter})) \quad
    \hat{r}^{\text{max}} = \max_{\p_i\in \PC_\barrel} (\hat{\eaxis}\cdot(\p_i-\hat{\ecenter}))\nonumber
\end{gather}

Since the extrusion axis we estimate is unoriented, we compute the range by taking the maximum of the absolute values from~\cref{eq:gcest}, and extruding that computed extent in both directions, \ie $|\hat{r}_k^{\text{min}}|=|\hat{r}_k^{\text{max}}|, \hat{r}_k^{\text{max}}=-\hat{r}_k^{\text{min}}$.
{Note that, all of these operations lend themselves to differentiation. This will be useful in the following section, where we learn to decompose point clouds into extrusion cylinder primitives.}

\section{Reverse Engineering Extrusion Cylinders}

\paragraph{Problem setting and overview} 
We assume that we observe an input point cloud $\PC\in\Rbb^{N\times 3}$. Our goal is to decompose its underlying geometry into a set of extrusion cylinders $\{ \gc_1, \gc_2, ..., \gc_K \}$.
We propose to solve this sketch-extrude decomposition problem in a geometry-grounded way by first learning the underlying geometric properties as a proxies. 
Specifically, these proxies include (i) an \emph{instance segmentation} $\hat{\W}\in\Rbb^{N\times K}$ as a per-point membership defining the likelihood of assigning each point to a certain segment $k\in\{1\dots K\}$, where each segment is an extrusion cylinder (ii) a \emph{base-barrel segmentation} $\hat{\B}\in [0,1]^{N\times 2}$, instantiating to $\hat{\B}_{:,0}==\mathbf{0}$ for barrel points, $\hat{\B}_{:,1}==\mathbf{1}$ for the base, and (iii) the surface normals ($\hat{\N}\in\Rbb^{N\times 3}$).  

In particular, we model our function approximator $\network_{\pars}:\PC\mapsto (\hat{\M},\hat{\N})$ as a neural network, whose architecture will be precised in~\cref{sec:arch}. 
Each entry in the output $\hat{\M}\in\R^{N\times 2K}$ indicates the likelihood of each point to belong to the base or the barrel of a particular segment. 

\subsection{Inferring \primname Parameters}
Now given predicted geometric proxies $(\hat{\M},\hat{\N})$, we establish a differentiable and closed-form formulation to estimate other extrusion parameters(\cref{sec:gencylinder}). 
$\hat{\M}$ compactly and jointly combines the predicted probability of a point 1) being either a base or a barrel, and 2) belonging to a certain segment. 
We then apply a row-wise softmax turning $\hat{\M}$ into a row-stochastic matrix whose $i^\text{th}$ row indicates the belonging of point $\mathbf{p_i}$ to one of the $2K$-classes, \ie
$\hat{\M}_{i,:}=(\mathbb{P}(\mathbf{p_i}\in \mathbf{P}^0_\barrel),\mathbb{P}(\mathbf{p_i}\in \mathbf{P}^0_\base),\dots,\mathbb{P}(\mathbf{p_i}\in \mathbf{P}^K_\barrel),\mathbb{P}(\mathbf{p_i}\in \mathbf{P}^K_\base))$~\footnote{$\mathbb{P}$ here denotes probability, and superscript denotes the segment index.}.
We can then recover $\hat{\W}$ by summing every two consecutive columns \ie $\hat{\W}_{:,j} = \hat{\M}_{:,2j} + \hat{\M}_{:, 2j+1}\,\forall j$. Similarly, $\hat{\B}$ can be obtained by summing up all odd/even columns of $\hat{\M}$, \ie $\hat{\B}_{:,0} = \sum_i \hat{\M}_{:,2j}$ and $\hat{\B}_{:,1} = \sum_i \hat{\M}_{:, 2j+1}$. 
\vspace{-0.1cm}
\begin{thm}\label{thm:M}
Matrix $\hat{\M}$ cannot be uniquely recovered from $\hat{\W}$ and $\hat{\B}$.
\end{thm}
\vspace{-0.1cm}
We refer the reader to our supplementary material for the full proof.
\cref{thm:M} guarantees that $\M$ is a rather compact parameterization of the number of unknowns \ie it does not suffice to learn $\W$ and $\B$, individually. Once predicted, we can directly use $\hat{\M}$ in order to solve for the parameters of each extrusion cylinder. To this end, we set:
\begin{align}
    \Weights_\barrel^k = \text{diag}(\hat{\M}_{:,2k})\quad 
    \Weights_\base^k = \text{diag}(\hat{\M}_{:,2k+1}),
\end{align}
$\forall k\in[1.K]$, and use them (along with the surface normals $\hat{\N}$ as the weights for the algorithm presented in~\cref{thm:weightedEAxis} to obtain $\hat{\eaxis}_k$ for each segment $k$\footnote{The weights used here are the probabilities that each point is a base/barrel point of extrusion segment k}. Then, the rest of the parameters are estimated individually for each segment as explained in~\cref{sec:gencylinder}, \eg~\cref{eq:gcest}. 

\paragraph{Inferring sketches} Finally, we show how to predict the sketch representation $\profile$. We first project the barrel points of each segment $k$ onto the plane defined by $(\hat{\ecenter}_k, \hat{\eaxis}_k)$ and scale by $\hat{s}_k$. This results in a 2D point cloud $\hat{\sketch}_k\in\Rbb^{N_\barrel^k\times 2}= s_k\Pi(\hat{\mathbf{P}}^k_\barrel,\hat{\eaxis}_k,\hat{\ecenter}_k)$\footnote{Similarly, we project the 3D normals of each point onto the same sketch plane endowing the 2D sketch with 2D normals.}. This creates a representational discrepancy when compared to~\cref{dfn:sketch} where the loops must be closed. A naïve approach could attempt to summarize the points into 2D primitives \eg line segments and arcs. However, (i) guaranteeing that the output sketch is closed and non-self intersecting is hard, and (ii) there can be multiple approximations of primitives for the same 2D sketch, hampering the learnability. Instead, inspired by the recent deep implicit works~\cite{Park_2019_CVPR,atzmon2020sal}, we represent the sketch implicitly, by learning the parameters $\parsIGR$ of an \emph{encoder function} $f_{\parsIGR}(\hat{\sketch}_k)\in \Rbb^D$ that maps the 2D point cloud into a global, normalized \emph{sketch latent space}. 
This latent code acts as the condition of a \emph{decoder} $\mathcal{S}:(\Rbb^D\times\Rbb^2)\rightarrow\Rbb$ mapping $(\mathbf{r} \in \Rbb^2)$ to its signed distance value to the underlying normalized sketch $\profile_k$: $d(\profile_k,\mathbf{r})\approx\mathcal{S}(f(\hat{\sketch}_k), \mathbf{r})$. Here, $d$ is the distance between $\mathbf{r}$ and ground truth \emph{sketch curve} of the segment $\profile_k$, approximated by implicit function $\mathcal{S}$ that is conditioned on the encoded 2D \emph{point cloud sketch approximation} $f(\hat{\sketch}_k)$. During inference, marching squares can be ran on the resulting field to obtain the closed sketch curves. The training of $f$, $\mathcal{S}$ and $\network_{\pars}$ will be specified below.

\subsection{Training}
\label{sec:train}
Assuming the availability of ground truth (GT) labels, we train the parameters ${\pars}$ of $\network_{\pars}$ using a multi-task, non-convex objective composed of segmentation (seg), base-barrel classification (bb), normal (norm) and sketch regularization losses:
\begin{equation}
    \loss = \loss_{\text{seg}} + \lambda_{\text{bb}}\loss_{\text{bb}} + \lambda_{\text{norm}}\loss_{\text{norm}} + \lambda_{\text{sketch}}\loss_{\text{sketch}},
    \label{full_loss}
\end{equation}
where $\lambda_{\text{bb}}=\lambda_{\text{norm}}=\lambda_{\text{sketch}}=1$. 
In what follows, we define and detail each of the loss terms. 

\paragraph{Normal estimation loss ($\loss_{\text{norm}}$)} We begin by estimating per-point unoriented normals $\hat{\N}^{N\times 3}$ 
for the given input point cloud $\PC$. 
and penalize for absolute cosine distance between the predicted and GT normals:
\vspace{-0.1cm}
\begin{equation}\label{eq:normals}
    \loss_{\text{norm}} = \sum\limits_{i=1}^{N} (1-|\hat{\N}_{i,:}^\top\N_{i,:}|).
\end{equation}
\vspace{-0.1cm}
\paragraph{Extrusion cylinder segmentation loss ($\loss_{\text{seg}}$)}
Points belonging to the same extrusion cylinder segment are assumed to be created by the same sketch-extrusion block. 
Note that, our problem does not admit a unique solution to the ordering of the segments as there could be multiple possible orderings of the sketch-extrude blocks that can yield the same output geometry. We instead predict a set of unordered segments and use Hungarian matching to find the best one-to-one matching with the ground truth extrusion cylinder segments, by computing Relaxed Intersection over Union (RIoU) between the predicted and ground truth segmentation:

\begin{equation}
\text{RIoU}(\U, \V)=\frac{\U^\top\V}{||\U||_1+||\V||_1 - \U^\top\V}.    
\end{equation}
Our final loss then maximizes the mean RIoU between the predicted and ground truth extrusion cylinder segments:
\begin{equation}
    \loss_{\text{seg}} = \frac{1}{N} \sum_{i=1}^{N} (1 - \text{RIoU}(\hat{\W}_{i,:}, \W_{i,:})).
\end{equation}
This also maintains the pipeline to be differentiable almost everywhere, and hence can be trained end-to-end. 
From here on, our notations assume a reordering such that the predicted and ground truth extrusion cylinders are in correspondence.

\paragraph{Base-and-Barrel classification loss ($\loss_{\text{bb}}$)}
Retrieving $\hat{\B}$ from $\hat{\M}$ allows us to further impose a binary cross-entropy loss between all the base-barrel memberships and their corresponding GT labels: 

\begin{equation}
    \loss_{\text{bb}} = \sum_{i=1}^{N} \sum_{j=1}^2 \hat{\B}_{i,j}\log(\B_{i,j}). 
\end{equation}


\paragraph{Sketch consistency loss ($\loss_{\text{sketch}}$)}
To ensure that the predicted parameters can yield meaningful sketches, we introduce the regularizer $\loss_{\text{sketch}}$. We first describe the the training process for $f_{\parsIGR}$ and $\mathcal{S}$ before defining $\loss_{\text{sketch}}$.

Since it is difficult to obtain direct supervision for GT signed distance fields, we propose to learn the implicit curve representation directly from \emph{raw} 2D point clouds. We use ground truth barrel points for each segment ($\mathbf{P}_\barrel^k$) of the models and its extrusion parameters to obtain clean sketches $\mathbf{S}_k\in\Rbb^{N_k\times 2}= s_k\mathbf{\Pi}(\mathbf{P}_\barrel^k,\eaxis_k,\ecenter_k))$ to train $f_{\parsIGR}$ and $\mathcal{S}$. Yet, our sketches can be created by positive or negative operations and we cannot uniquely determine the sign of the surface normals. Thus, we use an \emph{unsigned} variant of IGR~\cite{igr}, to learn $f_{\parsIGR}$ and $\mathcal{S}$. The loss for a sketch point cloud $\{\s_i \in\Rbb^2  \subset \sketch_k\}$ \footnote{Note that here we define the loss per sketch point cloud of an extrusion segment index $k$ denoted by $\sketch_k$, for the \emph{ground truth} projected point cloud and $\hat{\sketch}_k$ for the predicted projected point cloud.} with corresponding normals $\{\g_i\in\mathbb{S}^1\}$ is defined as 
\begin{equation}
    \mathcal{L}_\mathcal{S}  \triangleq \mathcal{L}_\text{manifold} + \lambda_1\mathcal{L}_\text{non-manifold},
\end{equation}
\noindent where
\begin{equation}
\begin{aligned}
    \mathcal{L}_\text{manifold}  &= \frac{1}{|\sketch_k|}\sum_{\s_i\in\sketch_k} (|\mathcal{S}(f_{\parsIGR}(\sketch_k), \s_i)| \\
    &+ \lambda_2\min(||\triangledown\mathcal{S}(f_{\parsIGR}(\sketch_k), \s_i) \pm n_i||) ).
\end{aligned}
\end{equation}
Note, the second term is different from the original IGR~\cite{igr} as it considers \emph{unoriented} normals. Moreover,
\begin{equation}
\begin{aligned}
    \mathcal{L}_\text{non-manifold}  = \mathbb{E}_{\mathbf{r}}(||\triangledown_{\mathbf{r}}\mathcal{S}(f_{\parsIGR}(\sketch_k), \mathbf{r})||-1), 
\end{aligned}
\end{equation}
denotes the Eikonal term and encourages the gradients to have a unit norm everywhere ($\forall \mathbf{r}\in\mathbb{R}^2$). Below assumes that $f_{\parsIGR}$ and $\mathcal{S}$ are trained to convergence, which then allows us to define our sketch regularization loss.

We make the observation that when the decomposition and hence extrusion parameters we predict are close to the GT, the latent embedding of the projected 2D point cloud $\hat{\sketch}_k = \hat{s_k}\mathbf{\Pi}(\hat{\PC}^k_\barrel,\hat{\eaxis}_k,\hat{\ecenter}_k)$ should be close to the GT sketch embedding $f_{\parsIGR}(\sketch_k)$, $\hat{\PC}^k_\barrel$ is obtained using $\hat{M}$. Under the light of this \emph{prior}, we train a joint embedding space $g$ together with $\network_{\pars}$ to match the sketch embedding $f_{\parsIGR}$.
\begin{equation}
    \mathcal{L}_\text{sketch} = ||g(\hat{\sketch}_k)-f_{\parsIGR}(\sketch_k)||_2.
\end{equation}
As $\mathcal{G}_\theta$ and $g$ are jointly optimized, this loss encourages $\mathcal{G}_\theta$ to predict extrusion cylinder decompositions that result in parameters that produce sketches close to GT.

\subsection{Network Details}\label{sec:arch}
We use a Pointnet++~\cite{qi2017pointnet++} backbone to learn a global 3D point cloud feature. This feature vector is then passed through two separate fully connected branches to obtain instance and base/barrel segmentations $\hat{\mathbf{M}}$ as well as normals $\hat{\mathbf{N}}$. Our network, $\network_{\pars}$ is first trained with segmentation ($\loss_{\text{seg}}$), base-barrel classification ($\loss_{\text{bb}}$) and normal losses ($\loss_{\text{norm}}$). 
We also separately pre-train the sketch implicit network $\mathcal{S}$ using as input, the ground truth point clouds projected  and scaled using the ground truth axis/scale for each separate extrusion segment of all the models. We use the same architecture as IGR~\cite{igr} for $\mathcal{S}$, and a PointNet~\cite{qi2017pointnet} encoder for $f_{\parsIGR}$ with latent dimension $D=256$.
We then append both pretrained networks together train our full model (\textbf{Point2Cyl}) using the full loss function $\loss$ with $\mathcal{L_\text{sketch}}$ defined in~\cref{full_loss}. $g$ was trained from scratch with the same architecture as $f_{\parsIGR}$. $f_{\parsIGR}$ and $\mathcal{S}$ are kept fixed during the final training. More details can be found in the supplementary.





\vspace{-0.1cm}
\section{Results}\label{sec:results}
\vspace{-3mm}
\paragraph{Datasets}
We adapt the two recent CAD datasets, Fusion Gallery~\cite{willis2020fusion} and DeepCAD~\cite{wu2021deepcad}, to show the feasibility of our approach.
We first preprocess the data to extract a set of extrusion cylinders with their parameters for each solid model. For each model, we randomly sample 8192 points from each underlying normalized and centered mesh to obtain input 3D point clouds. Ground truth sketches are represented as normalized 2D point clouds from randomly sampling 8192 points from the barrels of each extrusion cylinder. 
We train/test our networks on 4316/1242 models on Fusion and 34910/3087 models on DeepCAD. Details on the preprocessing, extrusion cylinder and sketch extraction are found in our supplementary material.

\begin{table*}[t]
\begin{minipage}[t][5cm]{0.6\textwidth}
    \centering
    \setlength{\tabcolsep}{0.3em}
    \caption{Quantitative results on Fusion Gallery and DeepCAD datasets.}
    \vspace{-0.3cm}
\newcolumntype{Y}{>{\centering\arraybackslash}X}
{\footnotesize
  \begin{tabular}{ll|c|c|c|c|c|c|c}
    \toprule
    &&Seg.$\uparrow$ & Norm.$(\degree)\downarrow$ & B.B.$\uparrow$ & E.A.$(\degree)\downarrow$ & E.C.$\downarrow$ &  Fit Cyl $\downarrow$ & Fit Glob$\downarrow$   \\
    \midrule
    \multirow{5}{*}{\rotatebox[origin=c]{90}{Fusion Gallery}}
    & H.V. + $\mathbf{N_J}$  & 0.409 & 12.264 & 0.595 & 58.868 & 0.1248  & 0.1492 & 0.0683  \\
    & D.P. & - & - & - & 30.147 & 0.1426 & 1.4132 & 0.4257\\
    & w/o $\mathcal{L}_\text{sketch} + \mathbf{N_J}$ & 0.699 &  12.264 & \textbf{0.913} & 14.169 & 0.0729 & 0.0828 & 0.0330\\
    & w/o $\mathcal{L}_\text{sketch}$ & 0.699 & 8.747 & \textbf{0.913} & 9.795 & 0.0727 & 0.0826 & 0.0352\\
    & Ours (\textbf{Point2Cyl}) & \textbf{0.736} &\textbf{ 8.547 }& 0.911 & \textbf{8.137} & \textbf{0.0525} & \textbf{0.0704} & \textbf{0.0305}\\
    \midrule
    \multirow{5}{*}{\rotatebox[origin=c]{90}{DeepCAD}}
    & H.V. + $\mathbf{N_J}$ & 0.540 & 13.573 & 0.577 & 59.785 & 0.0435 & 0.1664 & 0.0459  \\
    & D.P. & - & - & - & 48.818 & 0.0716 & 0.4947 & 0.4840\\
    & w/o $\mathcal{L}_\text{sketch} + \mathbf{N_J}$ & 0.829 & 13.573 & 0.916 & 10.109 & 0.0275 & 0.0856 & 0.0312\\
    & w/o $\mathcal{L}_\text{sketch}$ & 0.829 & 8.850 & 0.916 & 8.085 & 0.0273  & 0.0783 & 0.0324\\
    & Ours (\textbf{Point2Cyl}) & \textbf{0.833} & \textbf{8.563} & \textbf{0.919} & \textbf{7.923} & \textbf{0.0267} & \textbf{0.0758} & \textbf{0.0308}\\
    \bottomrule
  \end{tabular}
  }
\vspace{-0.5\baselineskip}
\label{tab:quantitative}
\end{minipage}\hfill
	\begin{minipage}[t]{0.37\textwidth}
	    \vspace{-1mm}
		\centering
		\includegraphics[width=0.8\linewidth]{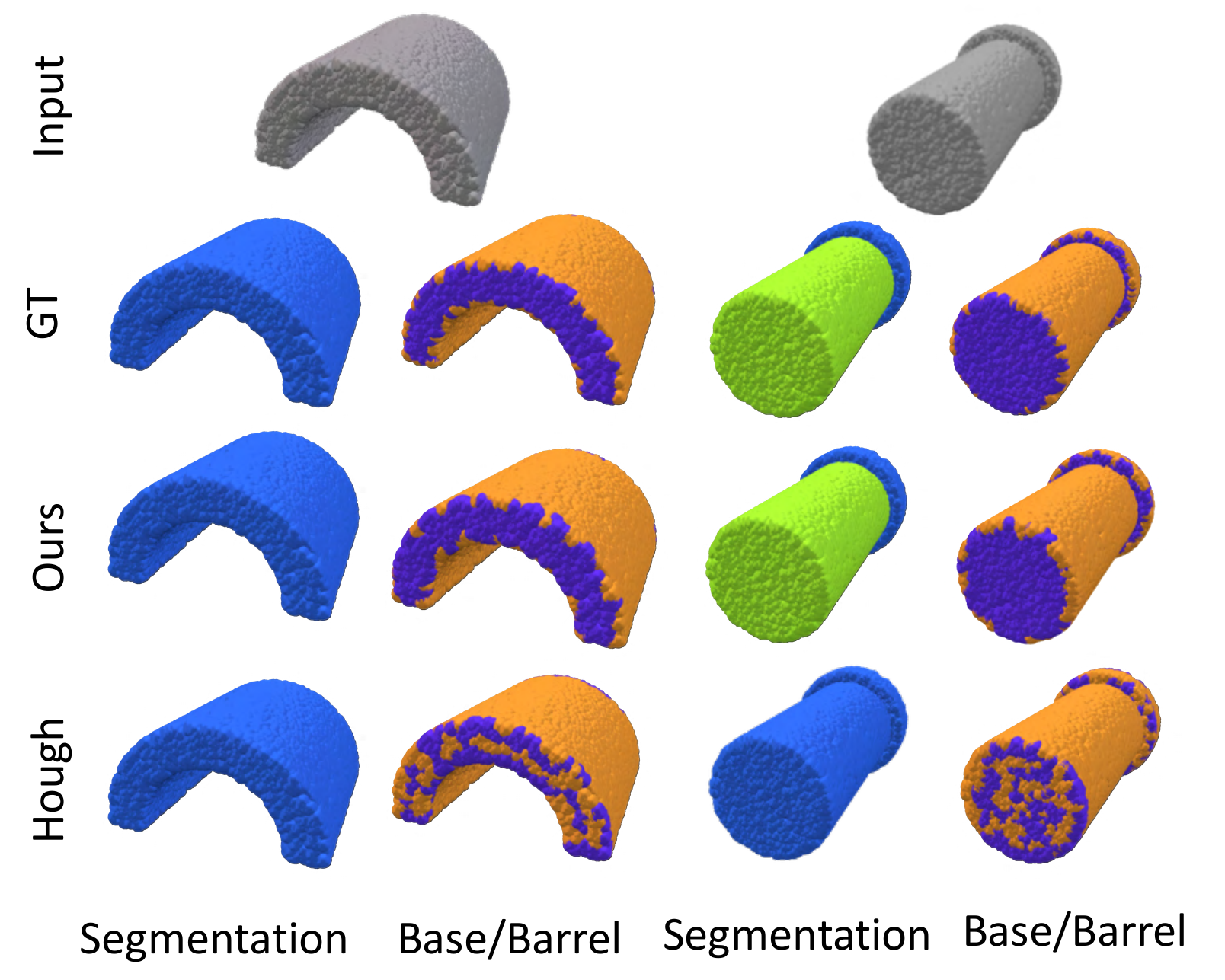}\vspace{-3mm}
    \captionof{figure}{Comparisons with Hough voting baseline. 
    }
    \label{fig:hough}
\end{minipage}
\vspace{-0.5cm}
\end{table*}

\paragraph{Evaluation metrics}
To the best of our knowledge, we are the firsts to address the problem at hand. Hence, we assess different aspects of our algorithm by introduction a diverse set of evaluation metrics, defined per model, as follows:
\begin{itemize}[leftmargin=\parindent,noitemsep,topsep=0.2em]
    \item{\textbf{Segmentation IoU (Seg.)}
    Assuming that the Hungarian matches reorders the extrusion segment into correspondence with the GT, we define the total RIoU loss: 
    $\frac{1}{K}\sum_{k=1}^{K}\text{RIoU}(\mathbbm{1}(\mathbf{\hat{W}}_{:,k}), \mathbf{W}_{:,k})$, where $\mathbbm{1}(\cdot)$ denotes a one-hot conversion.}
    \item{\textbf{Normal angle error (Norm.)}} We consider the agreement of GT and predicted normals:
    $\frac{1}{N}\sum_{i=1}^{N}\text{cos}^{-1}(|\mathbf{\hat{N}}_{i,:}^T\mathbf{N}_{i,:}|)$
    \item{\textbf{Base/barrel classification accuracy (B.B.)}} The amount of correctly predicted base/barrel labels reads:     $\frac{1}{N}\sum_{i=1}^{N} (\mathbbm{1}(\mathbf{\hat{B}_{i,:}})==\mathbf{B_{i,:}})$
    \item{\textbf{Extrusion-axis angle error (E.A.)}} defines the angular error between GT and predicted axes: $\frac{1}{K}\sum_{k=1}^{K}\text{cos}^{-1}(|\hat{\eaxis}_{k,:}^T\eaxis_{k,:}|)$
    \item{\textbf{Extrusion center error (E.C.)}} measures the distance to the GT center:  $\frac{1}{K}\sum_{k=1}^{K}||\hat{\ecenter}_{k,:}-\ecenter_{k,:}||_2$
    \item{\textbf{Per-extrusion cylinder fitting loss (Fit Cyl.}) measures how well the predicted unbounded extrusion cylinder parameters fit the GT extrusion cylinder segments:\vspace{-1mm}
    \begin{equation}
    \frac{1}{K}\sum_{k=1}^{K} \Big(\mathcal{F}\triangleq \sum_{\mathbf{s_i}\in \bar{\mathbf{S}}_k}  |\mathcal{S}(f_{\mathbf{\beta}}(\hat{\mathbf{S}}_k), \mathbf{s_i})|\Big),
    \end{equation}
    where $\bar{\mathbf{S}}_k = \hat{s}_k\mathbf{\Pi}(\mathbf{P}^{\text{barrel}_k}, \hat{\eaxis}_k, \hat{\ecenter}_k)$, which projects per-segment GT barrel points using predicted extrusion parameters. Let us call the inner summation $\mathcal{F}(\mathbf{P},k)$, which represents the goodness of fit for the $k^{th}$ extrusion.
    }
    \item{\textbf{Global fitting loss (Fit Glob.)} measures how at least one (any) of the predicted unbounded extrusion cylinders explain the input per each barrel point:\vspace{-1mm}
    \begin{equation}\frac{1}{|\mathbf{P}_\barrel|}\sum_{\mathbf{p}_i \in \mathbf{P}_\barrel} \min_{k=\{1, ...,K\}}|\mathcal{S}(f_{\mathbf{\beta}}(\hat{\mathbf{S}}_k), \mathbf{q}_i)|,
    \end{equation}
    where $\mathbf{q}_i = \hat{s}_k\Pi(\mathbf{p}_i , \hat{\eaxis}_k, \hat{\ecenter}_k)$.
    }
\end{itemize}
Final scores are the average of the metrics across all shapes.

\paragraph{Implementation and runtime} We implement our network using Pytorch and trained our models for $\sim300$ epochs/until convergence. Moreover, our approach has 3.6M parameters, and on a single Titan RTX GPU, it takes 0.41s for each training batch and 0.25s for a single model at inference time. 



\subsection{Baselines}
\paragraph{Classical approaches}
While the complexity of the extrusion axis decomposition problem makes it tedious to design handcrafted algorithms, we are still interested in seeing what our data-driven approach could offer over the classical methods. We first compare the point normals learned by \methodname with a strong geometric baseline \emph{jet fitting}~\cite{cazals2005estimating}, ($\mathbf{N_J}$).
We then use these jet-normals to vote for the extrusion axes via Hough voting (\textbf{H.V.}). To distinguish multiple primitives, we back the voting by a mean-shift mode estimation~\cite{cheng1995mean}. Eventually, this baseline produces both the predicted segments as well as their corresponding extrusion axes. Hence, we can compute the per-point base/barrel labels by taking the dot product between point normals and extrusion axes. The other extrusion cylinder parameters can then be computed as described in our method. 



\paragraph{Direct prediction (\textbf{D.P.})}
We introduce a second baseline, in which a deep neural network that is trained to \emph{directly predict} the set of extrusion cylinder parameters \emph{without} first learning our geometric proxies which are the per point segmentation and normals. 
As such, it lacks the appropriate geometric inductive biases present in our method.
\textbf{D.P.} yields $K$ sets of extrusion parameters $\{(\hat{\eaxis}_k), \hat{\ecenter}_k), \hat{s}_k), g_{\text{DP}}(\hat{\mathbf{S}_k}) \}$\footnote{$g_{\text{DP}}$ is analogous to $g$ in $\mathcal{L}_{\text{sketch}}$}, with a supervision signal per each parameter.
We use Hungarian matching on $\mathcal{F}(\mathbf{P}, k)$, as previously used in Fit Cyl., to sort the predictions in the same order as the ground truth. 

The details of both (i) deriving the Hough transform for our primitives and (ii) the architecture of the \textbf{D.P.} baseline can be found in the supplementary document. 


\begin{figure*}[t!]
    \centering
    \includegraphics[width=0.97\linewidth]{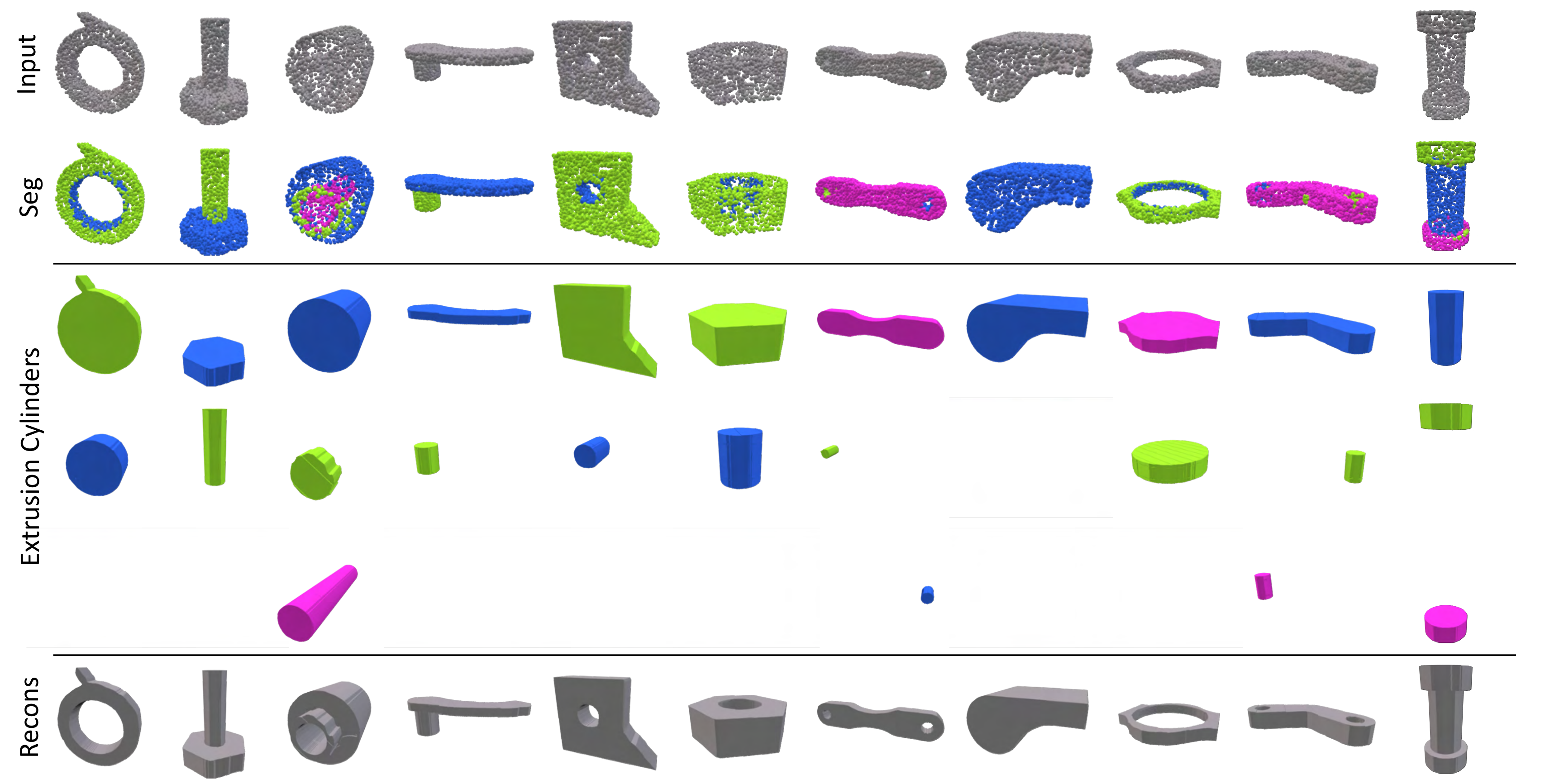}
    \caption{Qualitative examples for reconstruction. Figure shows (top-to-bottom) (1) input point clouds, (2) our predicted segmentation, (3-5) corresponding set of extrusion cylinders and (6) our final reconstruction. This figure also illustrates that individual extrusion cylinders from our decomposition result from a variety of closed loops. \vspace{-0.5cm}}
    \label{fig:reconstruction}
\end{figure*}

\subsection{Experimental Evaluation}
\paragraph{Quantitative evaluations}
\cref{tab:quantitative} shows our quantitative results on both the smaller Fusion Gallery and the big DeepCAD datasets. Our experiments demonstrate that our approach outperforms the other baselines (top two rows of each sub-block) across all metrics, often by significant margins, \ie $>15^\circ$ in E.A. Note that as the D.P. method does not produce intermediary proxies, we cannot evaluate those. 

In general, we see that even the direct prediction can work better than a simple handcrafted baseline (\textbf{H.V.}). This is further illustrated in~\cref{fig:hough} (left two columns) which shows a typical case where \textbf{H.V.} is able to vote for a reasonable extrusion axis, achieving reasonable base/barrel classification. However, as~\cref{fig:hough} (right two columns) illustrates, Hough cannot disambiguate segments that share the same extrusion axis. Such dependencies are hard to model. However, our learning-based approach can handle such cases thanks to its adaptivity to data. This further verifies the difficulty of manually modeling our problem. 

Additionally, our results affirm the validity of our geometric biases by revealing that predicting extrusion cylinder primitives through our proxy losses yields more faithful primitives compared to directly regressing the parameters as done by \textbf{D.P.}. 

\paragraph{Reverse engineering}
We propose two frameworks for reconstruction from extrusion cylinders. 
First, is a volume-based reconstruction engine, which builds a 3D signed distance volume by composing each of the extruded 2D sketch implicit fields from our prediction. The 3D transformation and scale are obtained from the corresponding extrusion cylinder parameters.
We also refine our segmentation prediction by a simple post-filtering and use robust methods in estimation of the scale and extent.
More on the post-processing details lie in the supplementary.~\cref{fig:reconstruction} shows representative examples reconstructed by our predicted extrusion cylinders using our volumetric reconstruction pipeline. The figure also presents the individual extrusion primitives, individually reconstructed. We can see that the predicted extrusion cylinders result from a variety of closed loops, showing that it is different from the standard primitive fitting works where the essential form of the primitive remains intact.

Our second reconstruction method exploits existing CAD modeller software such as Fusion360. To prepare the right input, we first extract the 2D sketch profiles by running marching squares on the predicted sketch implicits and then define the transformation matrices for each primitive using the extrusion cylinder parameters. Bottom-left most two shapes in~\cref{fig:teaser} depict such reconstructions.

\paragraph{CAD editing}
Once reconstructed using the latter technique, the outputs can be loaded into the existing CAD modellers for further editing such as adding fillets and chamfers. A sequence of such successful edits are shown in~\cref{fig:teaser} (bottom rows). As shown, we can create different interpretable shape variations from our extrusion cylinders by varying their parameters and/or their operation. Note that, exploring the space of such variations in a controllable manner would be challenging for any existing generative model such as~\cite{wu2021deepcad}. 


\paragraph{Ablation studies}
Finally, we ablate our approach by leaving out optional components such as the light sketch prior $g$ (w/o $\mathcal{L}_{\mathrm{sketch}}$) or data-driven normal estimation in addition to the sketch prior (w/o $\mathcal{L}_{\mathrm{sketch}}+\N_\mathbf{J}$) as shown in \cref{tab:quantitative}. 
These variants are trained similar to our original network to predict geometric proxies, but without applying the corresponding losses to backpropagate the gradients. Note that, without the sketch prior, the network can not regularize to produce segments that project close to a closed loop sketch, while not learning the normals cannot adaptively correct the issues or biases inherent in the hand-crafted estimation. For these reasons, both of these variants under-perform Point2Cyl. Though, they could still perform reasonable well when compared to the direct prediction or Hough baselines.


\section{Conclusion}\label{sec:conclude}
\vspace{-0.2cm}
We present \methodname, for reverse engineering 3D CAD models into primitives interpretable and usable by CAD designers. To solve this challenging discrete-continuous decomposition problem, we first introduced the \primname and developed its foundations for fitting to point sets. We then proposed differentiable algorithms suitable for a neural architecture, which partitions a point cloud into a set of \primnames. Our network benefited from a set of proxy predictions, which are shown to inject the correct geometric inductive biases. As opposed to standard primitive fitting, the output of \methodname allows for shape variations, and can be directly imported into existing CAD modellers for further reconstruction, visualization and re-usability.

\vspace{-0.1cm}
\paragraph{Limitations} Our approach 
does not make use of a known ordering for the decomposition or reconstruction. This circumvents assembly ambiguities and is beneficial for our training pipeline.
Furthermore, each extrusion cylinder primitive can either be a positive or negative volume with respect to the final shape. We also cannot disambiguate certain primitives \eg a box-like shape. 


\vspace{-0.1cm}
\paragraph{Acknowledgements} This work is supported by ARL grant W911NF-21-2-0104, a Vannevar Bush Faculty Fellowship, and gifts from the Autodesk and Adobe corporations.
M. Sung also acknowledges the support by NRF grant (2021R1F1A1045604) and NST grant (CRC 21011) funded by the Korea government(MSIT) and grants from the Adobe and KT corporations.

\bibliography{main}
\bibliographystyle{plain}

\renewcommand{\thesection}{A}
\renewcommand{\thetable}{A\arabic{table}}
\renewcommand{\thefigure}{A\arabic{figure}}

\clearpage
\section{Appendix}

\subsection{Extrusion Cylinders vs. Traditional Primitives}
\label{supp_primitives}
Our work focuses on decomposing an object into \emph{extrusion cylinders}, which are very common building blocks in CAD design, covering $>80\%/70\%$ of the faces in existing datasets~\cite{koch2019abc, brepnet}.
Extrusion cylinders are described by any arbitrary closed loop and hence do not assume fixed/regular geometry as in existing work that handle traditional \emph{\textbf{discrete/fixed}} primitives. Thus, they are not detectable by existing traditional primitive fitting works. 
For comparison, we ran a pre-trained model of SPFN~\cite{spfn} on our test set, which resulted to a fitting loss of 0.1111, compared to our 0.0305 in the main paper. 
By our primitive definition, our method alone cannot be used to recover primitives such as spheres and cones. An immediate solution is to integrate our method with existing primitive fitting work, \eg SPFN~\cite{spfn}, to jointly handle diverse types of primitives.
We also make the distinction that our work recovers primitive \emph{volumes} instead of surfaces, from works such as SPFN~\cite{spfn} that requires further stitching and cannot be used directly in boolean operations.
\begin{wrapfigure}{r}{0.06\textwidth}
\vspace{-5.75mm}
\hspace{-5.5mm}
\hspace*{-3mm}\includegraphics[width=0.1\textwidth]{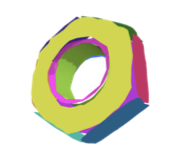}
\vspace{-5mm}
    \label{fig:cylinders}
\end{wrapfigure}
See illustration of a nut on the right in Fig.~\ref{fig:realscans_failurecases}-b. 
In SPFN, the single extrusion will be represented with nine \emph{surfaces}: six side planes, two top/bottom planes, and one inner cylinder.

\subsection{Theorem Proofs and Robustness}
\label{supp_theorems}
\subsubsection{Proof of Theorem 2}

We start by a short summary of Theorem 2 of the main paper:
\setcounter{thm}{1}
\begin{thm}[Weighted recovery of \textbf{extrusion axis} from points]
For a general weighted point set, the optimal \textbf{extrusion axis} is given by $\hat{\eaxis}=\argmin_{\eaxis, ||\eaxis||=1}(\eaxis^\top \Hm_{\Weights} \eaxis)$, where:
\begin{equation}\label{eq:eigen}
\Hm_{\Weights} = \N^\top\Weights_\barrel^\top\Weights_\barrel\N-\N^\top\Weights_\base^\top\Weights_\base\N.
\end{equation}
where $\Weights_\barrel = \mathrm{diag}(\weights_\barrel), \Weights_\base= \mathrm{diag}(\weights_\base) \in\R^{N\times N}$. $\weights_\barrel$/$\weights_\base$ indicate the barrel/base weights assigned to all points, respectively. 
\end{thm}
\begin{proof}
First, we use the observation in Eq. 1 of the main paper to formulate an objective $\loss_{\eaxis}$, whose minimum is attained at the point where $\eaxis$ is the best fitting extrusion axis. Next, we show that $\loss_{\eaxis}$ can be re-organized absorbing the weights into surface normals. Finally, we re-arrange the result into matrix form to obtain $\Hm_{\Weights}$. These steps read: 
\begin{align}\label{eq:Lreorg}
    \loss_{\eaxis} &= \sum_{j\in\mathrm{barrel}} w_j^2 (\n_j^\top \eaxis)^2 - \sum_{i\in\mathrm{base}} w_i^2 (\n_i^\top \eaxis)^2\\
    &= \sum_{j\in\mathrm{barrel}} (w_j\n^\top_j\eaxis)^\top (w_j\n^\top_j\eaxis) - \sum_{i\in\mathrm{base}} (w_i\n^\top_i\eaxis)^\top (w_i\n^\top_i\eaxis)\nonumber\\
    &= \sum_{i=1}^N (\phi_i^\barrel\n^\top_i\eaxis)^\top (\phi_i^\barrel\n^\top_i\eaxis) - (\phi_i^\base\n^\top_i\eaxis)^\top (\phi_i^\base\n^\top_i\eaxis)\nonumber\\
    &= \sum_{i=1}^N \eaxis^\top (\n_i^\top\phi^{\barrel^\top}\phi^{\barrel}_i\n_i-\n_i^\top\phi^{\base^\top}\phi^{\base}_i\n_i) \, \eaxis\\
    &= \eaxis^\top \Hm_{\Weights} \eaxis \label{eq:thm3H}
\end{align}
and therefore $\argmin_{\eaxis} \loss_\eaxis = \argmin_{\eaxis}\eaxis^\top \Hm_{\Weights} \eaxis$.
Here, scalars $w_i$ and $w_j$ can be viewed as matrices of size $1\times 1$. The second equality follows from gathering all the weights into $\weights$ such that $\phi^\barrel_i=0$ whenever $i$ indicates a base point, and vice versa. We collect those into vectors $\weights_\barrel=\{\phi_i^\barrel\}$, $\weights_\base=\{\phi_i^\base\}$.~\cref{eq:Lreorg} clearly shows that the weights that control the contribution of each normal can be absorbed into the normals themselves. The last step in~\cref{eq:thm3H} develops by virtue of this, \ie we can define a \emph{weighted normal} \ie $\bar{\n}_i\triangleq \n_i \phi^\barrel_i$ and exploit Thm. 1 to write:
\begin{align}
\Hm_{\Weights}&=(\bar{\N}_\barrel^\top\bar{\N}_\barrel-\bar{\N}_\base^\top\bar{\N}_\base)\\
&=(\Weights_\barrel\N)^\top(\Weights_\barrel\N)-(\Weights_\base\N)^\top(\Weights_\base\N)\\
&= \N^\top\Weights_\barrel^\top\Weights_\barrel\N-\N^\top\Weights_\base^\top\Weights_\base\N.
\end{align}
The second equality follows from individually weighting the normals using $\Weights_\barrel = \mathrm{diag}(\weights_\barrel), \Weights_\base= \mathrm{diag}(\weights_\base) \in\R^{N\times N}$.
Note that this constructive approach also presents a perspective on the energy induced by our eigenvector problem. As $\|\eaxis\|=1$ is desired,
the solution is given by the eigenvector corresponding to the smallest eigenvalue of $\Hm_{\Weights}$.
\end{proof}
On a closer look, this approach resembles a graph partitioning where the cost is defined over the surface normals. We leave further analysis as a future work.

\begin{figure}
    \centering
    \includegraphics[width=0.65\linewidth]{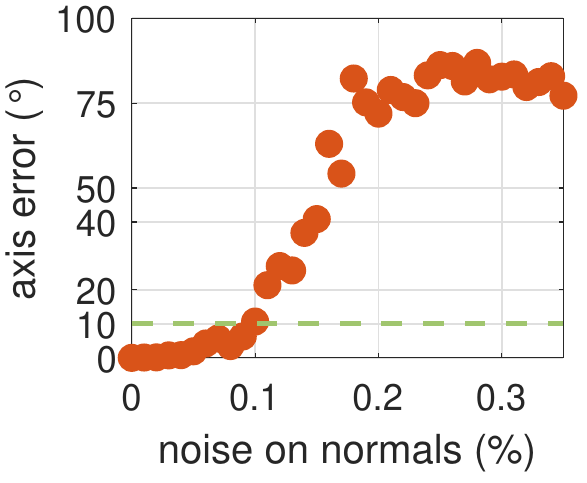}
    \caption{Robustness under noisy point normals.}
    \label{fig:noise}    
\end{figure}

\begin{figure*}
    \centering
    \includegraphics[width=\linewidth]{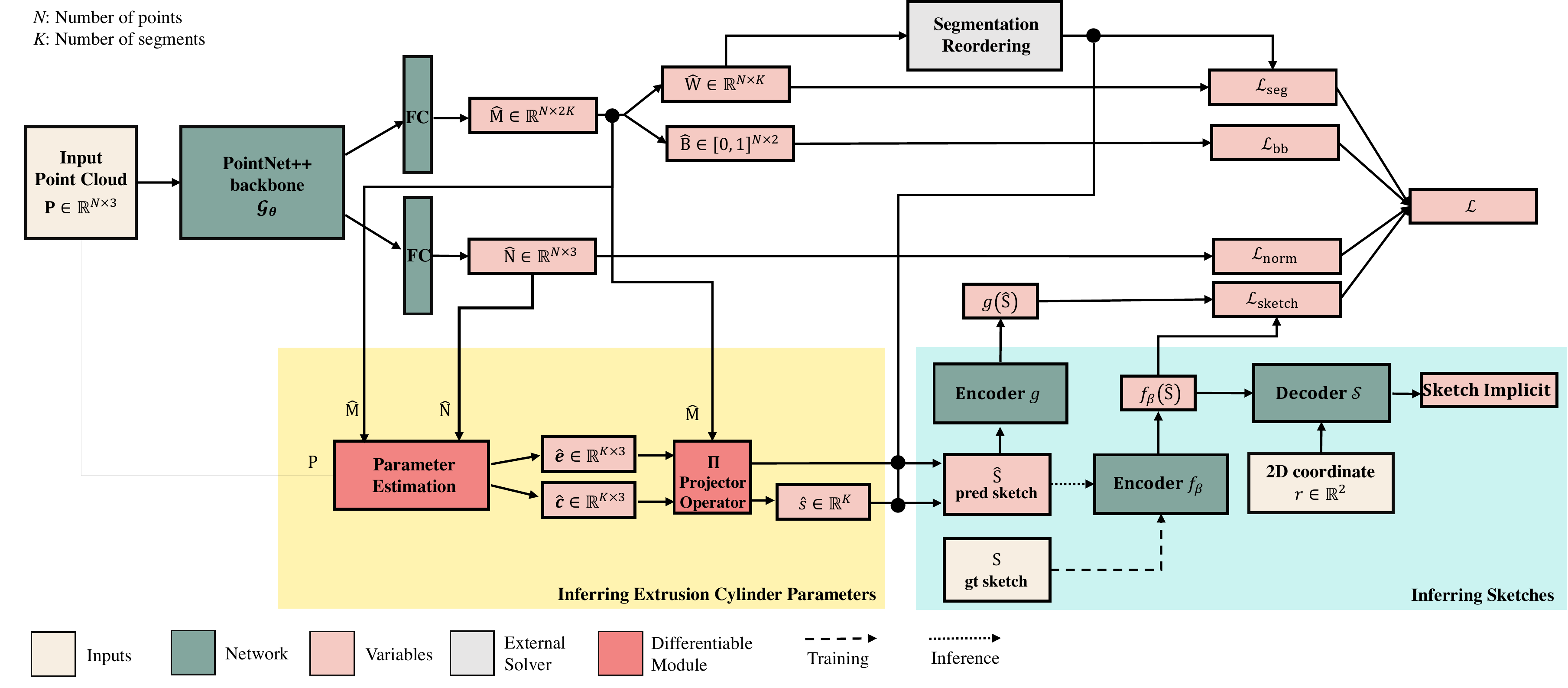}
    \caption{Network architecture of our Point2Cyl.}
    \label{fig:network}
\end{figure*}

\subsubsection{Proof of Theorem 3}
Ideally, we would like to predict a minimal set of parameters for obtaining $\hat{\M}$, which contains $N\times 2K$ parameters. An obvious question arises when we consider the two outputs obtained through $\hat{\M}$: $\hat{\W}$ and $\hat{\B}$. In total these two matrices contain $N\times (K+2)$ unknowns. From this lens, it seems tempting to predict $\hat{\W}$ and $\hat{\B}$ instead. We now re-state Theorem 3 of the main paper arguing that knowing $\hat{\W}$ and $\hat{\B}$ is not sufficient to analytically compute $\hat{\M}$:

\begin{thm}
Matrix $\hat{\M}$ cannot be uniquely recovered from $\hat{\W}$ and $\hat{\B}$.
\end{thm}
\begin{proof}
We begin by assuming that recovering $\hat{M}\in\Rbb^{N\times 2K}$ from $\hat{W}\in\Rbb^{N\times K}$ and $\hat{B}\in\Rbb^{N\times 2}$ would be possible. Hence, we assume the availability of $(\hat{W},\hat{B})$, while $\hat{M}$ remains unknown and make use of the constraints at our disposal. First, we organize the relationship of $\hat{W}$ and $\hat{B}$ to $\hat{M}$ (as given in the main paper, Sec. 4.1) into matrix notation:
\begin{align}
    \hat{W} = \hat{M} \Delta_W\qquad , \qquad    \hat{B} = \hat{M} \Delta_\B
\end{align}
where: 
\begin{align}
\Delta_W=\Eye_{K\times K} \otimes \begin{bmatrix} 1\\1 \end{bmatrix}\quad,\quad
\Delta_B=\one \otimes \Eye_{2\times 2}.
\end{align}
$\Eye_{N\times N}$ denotes an $N\times N$ identity matrix and $\otimes$, the Kronecker product. Below, we show these matrices for the case of $K=3$:
\begin{align}
    \Delta_W^{K=3} = \begin{bmatrix}
    1 & 0 & 0 \\
    1 & 0 & 0 \\
    0 & 1 & 0 \\
    0 & 1 & 0 \\
    0 & 0 & 1 \\
    0 & 0 & 1 \\
    \end{bmatrix} \quad, \quad
    \Delta_B^{K=3} = \begin{bmatrix}
    1 & 0\\
    0 & 1\\
    1 & 0\\
    0 & 1\\
    1 & 0\\
    0 & 1\\
    \end{bmatrix}
\end{align}
From the structure of $\M$, we also have $\M\one = \one$.
Note that, $\Delta_W$ and $\Delta_B$ are both non-square and hence non-invertible\footnote{Using pseudoinverse does not give us the exact solution in this case.}. However, it is of interest to see whether these two constraints would be simultaneously sufficient to recover $\hat{\M}$. To this end, we convert the observations thusfar into the following constraints:
\begin{align}
    \|\hat{\M}\Delta_W - \hat{\W}\| \to \min \quad \|\hat{\M}\Delta_B - \hat{\B}\| \to \min.
\end{align}
To ensure those, we minimize a loss $\loss_\M$:
\begin{align*}
    \loss_\M = \|\hat{\M}\Delta_W - \hat{\W}\|^2 + \|\hat{\M}\Delta_B - \hat{\B}\|^2 + \|\M\one-\one^\prime\|^2,
\end{align*}
by letting $\nabla_\M\loss_\M=0$. After simple derivations and a re-arrangement, this yields:
\begin{align}
    \M (\Delta_W\Delta_W^\top+\Delta_B\Delta_B^\top+\one\one^\top) =\hat{\W}\Delta_W^\top+\hat{\B}\Delta_B^\top +\one^\prime \one^\top \nonumber
\end{align}
where $\one$ and $\one^\prime$ denote one-vectors of different lengths. 
Note that the matrix $\D \triangleq (\Delta_W\Delta_W^\top+\Delta_B\Delta_B^\top+\one\one^\top)$ is known in advance and can be pre-computed. However, its rank will always be $\mathrm{rank}(\D)= K+1$ and therefore the number of linearly independent equations is insufficient to solve for $\M$. This concludes our proof that without further regularity, $\hat{\M}$ cannot be uniquely recovered from $\hat{\W}$ and $\hat{\B}$.
\end{proof}

\subsubsection{Robustness under Noisy Point Normals}
We include an experiment where the normals of the extrusion in Fig. 4 (main paper) is perturbed by an increasing amount of Gaussian noise. As shown on Fig.~\ref{fig:noise}, our fit could tolerate (error $<10^\circ$) $\sim 10\%$ noise on the normals, which is a realistic setting.
\subsection{Additional Details for Our Point2Cyl}
\label{supp_point2cyl}
We provide additional implementation details for our Point2Cyl. \cref{fig:network} shows our overall pipeline. We randomly sample 8192 points for the point cloud inputs for networks taken from the underlying mesh for each model. We set $K=8$ as the maximum number of extrusion segments and trained with a batch size of 4. When the number of extrusion cylinder segments in the ground truth label is less than $K$, the additional segments not matched with any of that in the ground truth are not included in the loss calculation.
$\mathcal{S}$ was trained with normalized 2D point clouds with 2048 points, with a batch size of 8 with $\lambda_1 = 0.1, \lambda_2 =1$. We use initial learning rate of 0.001 with a decay of 0.7. All models are trained for around 300 epochs or until convergence using the Adam optimizer. 

We further clarify that our predicted extrusion cylinders can either be positive/negative volumes, and it can be inferred at post-processing. One can project (oriented) 3D point normals onto the sketch plane and compare it with the gradient of the predicted sketch implicit function to determine additive/subtractive cylinders.

\subsection{Implementation Details of Baselines}
\label{supp_baselines}

\begin{figure*}[t!]
    \centering
    \includegraphics[width=0.8\linewidth]{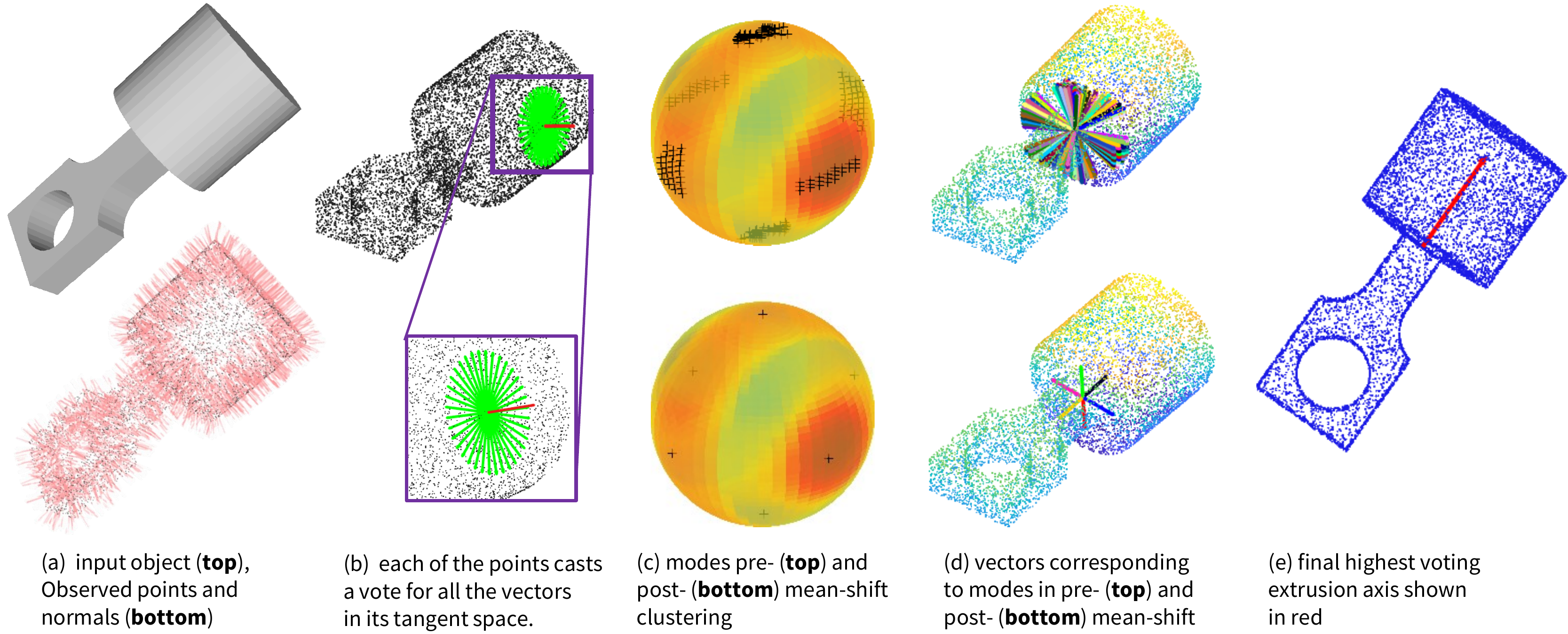}
    \caption{Steps of our Hough voting baseline for extrusion axis detection. Note that, our Point2Cyl algorithm can outperform this baseline. Nevertheless, it is still visible that this baseline can also tolerate clutter and confusing structures to a certain extent.}
    \label{fig:hough}
\end{figure*}
\subsubsection{Hough Voting for \primnames}
While Hough voting and its generalized variants are used in simple primitive detection~\cite{birdal2019generic,sommer2020primitect,Schnabel2007} or general object detection~\cite{drost2010model,birdal2015point,tejani2014latent}, their readily available adaptation to our problem remains unexplored. Hence, we create our Hough baseline by proposing a novel voting strategy for detection of extrusion cylinders.

In particular, we would like to model the \emph{likelihood} of an extrusion axis given all the surface normals in the data.
For an extrusion hypothesis, this reads:
\begin{align}
    p(\eaxis\,|\, \n_1,\dots,\n_N) &= p(\eaxis) \prod\limits_{i=1}^N p(\n_i,\eaxis)\\
    &= \frac{p(\eaxis) \prod\limits_{i=1}^N p(\n_i\,|\,\eaxis)}{p(\n_1,\dots,\n_N)}\\
    &= \alpha\prod\limits_{i=1}^N p(\n_i\,|\,\eaxis).
\end{align}
The final equality follows from the assumption of a uniform prior and constant normalizing factor. Taking the logarithm of both sides, we obtain:
\begin{equation}\label{eq:houghmle}
\log p(\eaxis\,|\, \n_1,\dots,\n_N) = \log\alpha + \sum\limits_{i=1}^N \log p(\n_i\,|\, \eaxis).
\end{equation}
At this stage we propose to model $p(\n_i\,|\, \eaxis)$ as a \emph{Gibbs measure}:
\begin{align}\label{eq:pnconde}
    p(\n_i\,|\, \eaxis) = \frac{1}{\beta} \exp(-\gamma v(\eaxis,\n_i)).
\end{align}
with $\beta$ being its normalizing constant. Plugging~\cref{eq:pnconde} into~\cref{eq:houghmle} leads to the Hough Transform for $\eaxis$,  $H(\eaxis)\triangleq\log p(\eaxis\,|\, \n_1,\dots,\n_N)$ :
\begin{align}
\log p(\eaxis\,|\, \n_1,\dots,\n_N) &= \log\alpha - N\log\beta + \sum\limits_{i=1}^N v(\eaxis,\n_i)\nonumber\\
H(\eaxis)&= c + \sum\limits_{i=1}^N v(\eaxis,\n_i)
\end{align}
where $c=\log\alpha - N\log\beta$ is a constant. Th  is suggest that the \emph{evidence} of an extrusion axis $\eaxis$ can be obtained by interpreted as the sum of the \emph{votes} cast per each normal $\n_i$.
Note that, in practice, leaving $c$ out leads to an \emph{unnormalized log probability distribution} over the extrusion axis given surface normals, \ie MAP (maximum a-posteriori) estimate trivializes to MLE (maximum likelihood estimation). This is what a Hough transform essentially computes.

The optionally \emph{fuzzy} voting function $v(\eaxis,\n_i)$ can be chosen to fit the geometric observation that the extrusion axis should lie in the tangent plane spanned by the surface normal, \ie it has to be \emph{normal to the normals}. With that, we can propose binary or non-binary voting functions:
\begin{align}
    v(\eaxis, \n_i) \triangleq \delta(|\eaxis_i^\top\n_i|) = \begin{cases}
    0, &|\eaxis_i^\top\n_i| < \epsilon \\
    1, &\text{\emph{o.w.}} \\
    \end{cases}
\end{align}
At this point, one can obtain a point estimate by:
\begin{equation}
    \eaxis^\star = \argmax_{\eaxis\in\mathbb{S}^2} H(\eaxis).
\end{equation}
However, for the cases where multiple modes as well as noise are present, such approach would not yield a robust estimate. Therefore, we shift our attention to \emph{mode seeking} and use the celebrated \emph{mean-shift algorithm}~\cite{cheng1995mean} to discover the modes of the underlying distribution. To this end, in practice, for each point, we maintain a set of random vectors lying in its tangent space defined by the surface normal. This defines an empirical/discrete scalar field over the \emph{Gaussian sphere} where each particle is a hypothesis for an extrusion axis. We then find the modes of this empirical distribution. 

Different stages of our Hough voting baseline are shown in~\cref{fig:hough}.

\begin{figure*}
    \centering
    \includegraphics[width=0.9\linewidth]{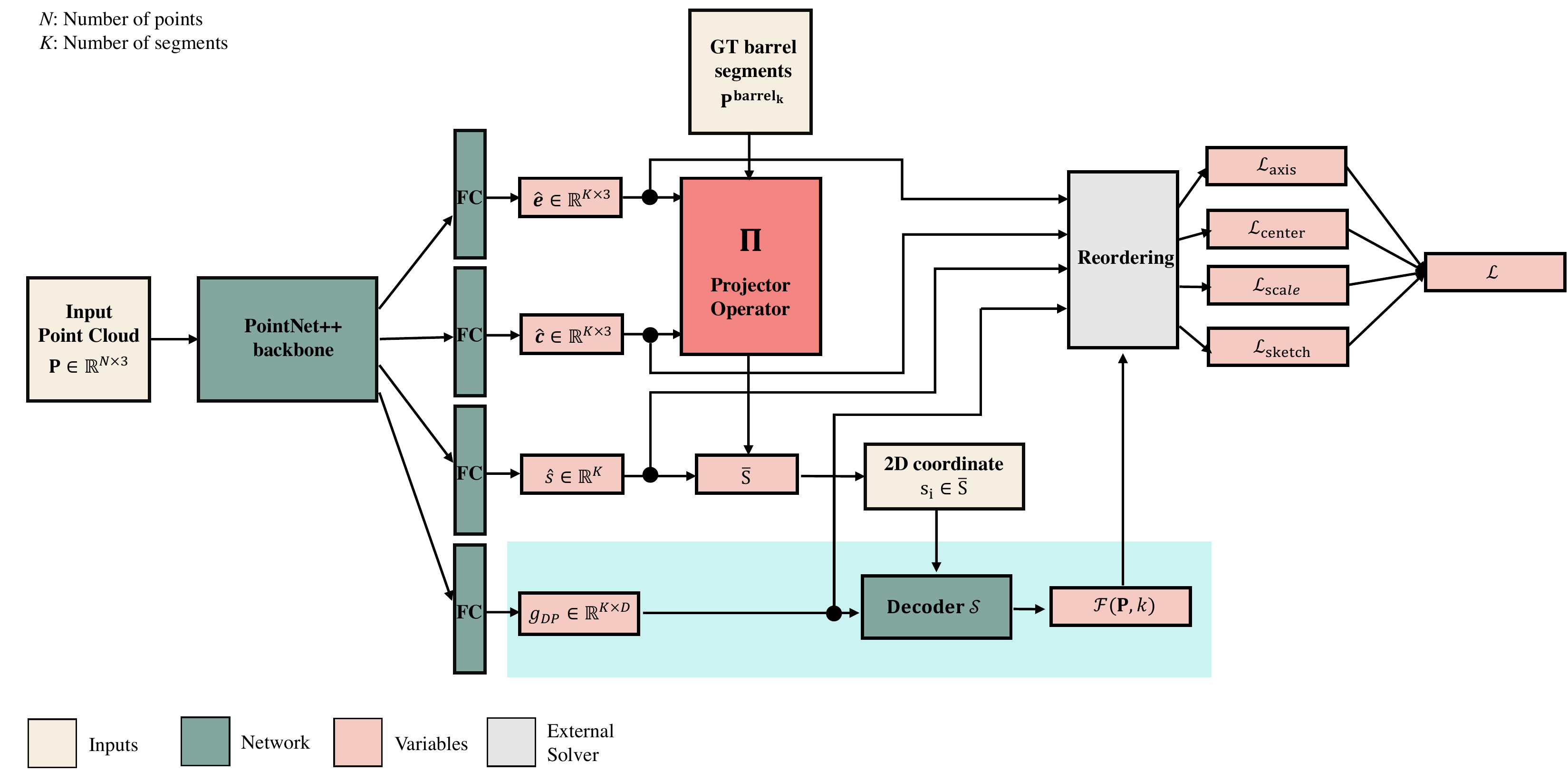}
    \caption{Network architecture of the direct prediction (D.P.) baseline.}
    \label{fig:dp_network}
\end{figure*}

\subsubsection{Direct Prediction (D.P.)}
We also provide additional details for \textbf{D.P.}, a baseline neural network that takes in point cloud $\mathbf{P}$ and directly predicts the extrusion parameters without segmentation into generalized cylinders. \cref{fig:dp_network} shows the network architecture for \textbf{D.P.}. We use a similar backbone as our network that encodes $P$ into a latent feature, and then we directly predict $K$ sets of extrusion parameters $(\hat{\mathbf{e}}, \hat{\mathbf{c}}, \hat{s}, g_{\text{DP}}(\hat{\mathbf{S}}_k))$ in four separate fully connected branches, where  $g_{\text{DP}}(\hat{\mathbf{S}}_k)) \in \mathbb{R}^D$ is the latent code that conditions the sketch decoder $\mathcal{S}$. To train the network, we project the barrel points for each ground truth segment $\mathbf{P}^{\text{barrel}_k}$ using each of the directly predicted extrusion parameters, similar to  $\mathcal{F}(\mathbf{P}, k)$ as introduced in Fit Cyl. in the Evaluation Metric section of the main paper. We use Hungarian matching on $\mathcal{F}(\mathbf{P}, k)$ to find correspondences with the ground truth extrusion parameters, and directly supervise each of the predicted parameters with their corresponding ground truth. We use the angular error (E.A.) for the extrusion axis, the distance/L2 loss for the extrusion center (E.C.), L1 for the extrusion scale, and $\mathcal{L}_{\text{sketch}}$ for $g_{\text{DP}}$. 


\subsection{Data Pre-Processing Details}
\label{supp_preprocessing}
We provide additional information on the data pre-processing used for both Fusion Gallery~\cite{willis2020fusion} and DeepCAD~\cite{wu2021deepcad}. We used the Reconstruction subset for Fusion Gallery~\cite{willis2020fusion}. For clarification on the ABC dataset~\cite{koch2019abc}, DeepCAD dataset~\cite{wu2021deepcad} is the subset of ABC dataset that collects the models constructed with sketch-extrude operations. Both datasets include json files that contain the step-by-step construction sequence and the resulting triangulated output model, represented as a mesh. We obtain the extrusion cylinder segmentation labels for each face of a model by tracking which construction operation in the json file created each of mesh faces in the final geometry. For each construction operation of each model, we also extract the corresponding extrusion axis from the json file.

We uniformly sample 8192 points over each model's surface as our input point cloud along with each point's corresponding normal and segmentation label from its mesh face, as previously described. The point clouds are also normalized to fit a unit sphere. We classify each point as base/barrel by checking the dot product between its normal and the extrusion axis of the extrusion cylinder segment it belongs to. We obtain the ground truth extrusion center by taking the mean of all barrel points from each segment. We also represent the sketch of each segment of the model with a point cloud by sampling 8192 points along the barrel faces, and projecting them onto the plane perpendicular to the corresponding extrusion axis, centering them with the extrusion center, and finally normalizing to fit a unit circle. 

We select a subset of all models that have $1-8$ extrusion cylinder segments. To balance our dataset, we also use a portion of the models with a single extrusion such that it only cover $\sim20\%$ of our data. To increase the number of training/test models, we also use the intermediate models in each construction sequence. Moreover, we discard models that had tapered extrusions, an extrusion segment with surface area $<2\%$ of the whole model, an extrusion segment that is too small whose extent is either too short $(<0.015)$ or had too few points $(<50)$.  We will release our processed data upon publication.

\subsection{Visualization Post-Processing Details}
\label{supp_postprocessing}
We also provide additional details on our post-processing step used for reconstruction refinement and the visualization of output models. 

We first refine the segmentation output of our network as follows: i) Using the initial predicted segmentation, we use DBSCAN to cluster the points that belong to the same segment. If a segment results in more than one cluster, we unlabel the points belonging to the smaller clusters. ii) An unlabeled point is labeled with the consensus of its neighbors. iii) A point is relabeled if its neighbors have a high consensus with a different label, and the neighborhood consensus is used as the new label of the point. 

We further use robust methods to estimate for scale and extent. To estimate for scale, we use RANSAC to randomly sample $1\%$ of the barrel points for each extrusion segment, which are then used to estimate for the scale, as described in our main paper. The scale estimate of the segment is accepted if it explains more than $80\%$ of all its barrel points. For extrusion extent, we use DBSCAN to cluster all the barrel points projected along the extrusion axis, and we calculate the extent based on the most dominant cluster.

Finally, we further optimize the sketch fitting by directly optimizing our sketch implicit network (modified from IGR~\cite{igr} as described in the main paper), to directly fit the projected barrel points based on our network's output for each individual segment.

\subsection{Additional Evaluations}
\label{supp_evaluations}

\subsubsection{Comparison with Conditional Generation Extension to DeepCAD~\cite{wu2021deepcad}}
We also analyze and compare our approach with a condition generation extension of DeepCAD~\cite{wu2021deepcad} that is cast as a future work in their paper. DeepCAD~\cite{wu2021deepcad} introduced a transformer-based generative model for CAD modeling sequences, similar to their proposed approach in their "Future Applications" section. We use the experimental code provided by the DeepCAD authors to encode point clouds using PointNet++ and map the resulting embeddings to the latent vector of their CAD sequence encoder from their original generative model.  The PointNet++ encoder was trained for 100 epochs with batch size of 128 on the same DeepCAD training subset as Point2Cyl. We use the Adam optimizer with initial learning rate of $10^{-4}$ and decay the learning rate by 0.1 every 30 epochs.


At inference time, we use the Pointnet++ encoder to get the latent embedding of the point cloud and the CAD sequence decoder to obtain the reconstruction. We use the same DeepCAD test subset as for Point2Cyl. We used the released implementation for DeepCAD sequence reconstruction and found that the modeller fails to reconstruct $11.5\%$ of the testing models from the output of the conditional DeepCAD generation. We report the fitting loss (Eq. 15 of the main paper) for DeepCAD (\textbf{0.0959}) vs. our model (\textbf{0.0758}) on the subset of models that~\cite{wu2021deepcad} was able to successfully reconstruct. We take the midpoint of the tokenized outputs from~\cite{wu2021deepcad} as the primitive parameters to obtain the corresponding extrusion cylinders that are used for loss computation. \cref{fig:supp_quali} shows some qualitative comparisons with reconstruction outputs from our Point2Cyl. These are examples where the DeepCAD reconstruction pipeline is able to output a solid. Results show that the outputs of DeepCAD does not always match the shape of the input geometry, and moreover, some examples show that DeepCAD often struggles to produce valid solid models in the output. 
In order to create a valid solid model, it is a requirement that the sketch profiles do not self-intersect.  When self-intersecting profiles are extruded, the resulting solids will not define a closed and watertight volume.  They will also fail the consistency checks of the solid modeling kernel and in some cases may fail to triangulate. Our Point2Cyl uses 2D implicits to define a sketch profile, and hence they will not intersect compared to DeepCAD, which can produce in self-intersecting geometry.
Also, our Point2Cyl is trained to minimize the fitting error, while the encoding-decoding architecture of DeepCAD is not trained to make the output \emph{fit} the input point cloud, and thus resulting in completely different shapes in some cases.


\begin{figure}[t]
\vspace{-2mm}
	\begin{center}
		\includegraphics[width=\linewidth]{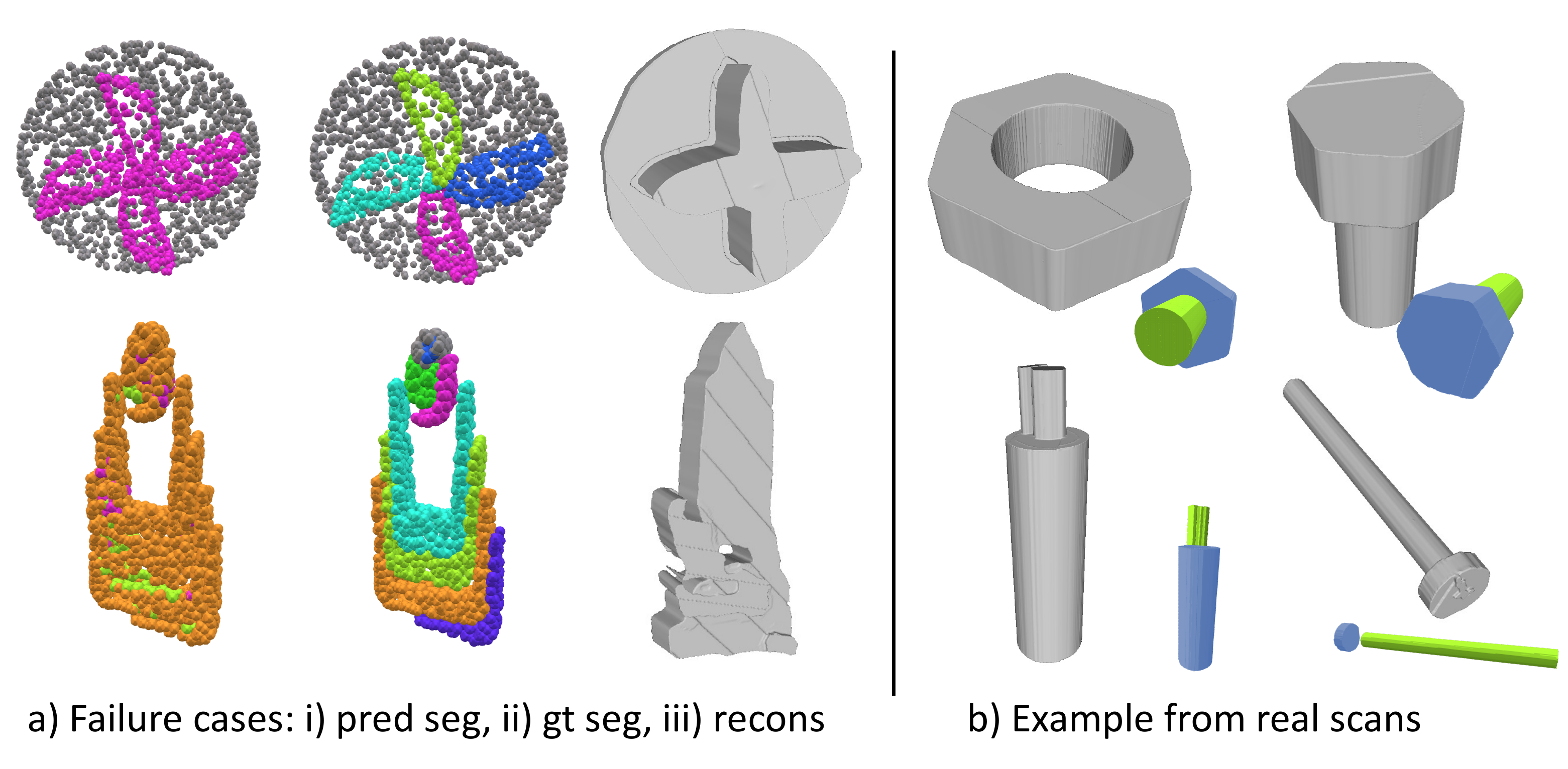}
	\end{center}
	\vspace{-0.5cm}
	\caption{a) Failure cases. b) Real scan results.}
	\vspace{-0.7cm}
	\label{fig:realscans_failurecases}
\end{figure}

\begin{table*}
    \centering
  \begin{tabular}{l|c|c|c|c|c|c|c|c}
    \toprule
    &$\sigma$&Seg.$\uparrow$ & Norm.$(\degree)\downarrow$ & B.B.$\uparrow$ & E.A.$(\degree)\downarrow$ & E.C.$\downarrow$ &  Fit Cyl $\downarrow$ & Fit Glob$\downarrow$   \\
    \midrule
    H.V. + $\mathbf{N_J}$  & &0.409 & 12.264 & 0.595 & 58.868 & 0.1248  & 0.1492 & 0.0683  \\
    D.P. & & - & - & - & 30.147 & 0.1426 & 1.4132 & 0.4257\\
    w/o $\mathcal{L}_\text{sketch} + \mathbf{N_J}$ & 0.00 &0.699 &  12.264 & \textbf{0.913} & 14.169 & 0.0729 & 0.0828 & 0.0330\\
    w/o $\mathcal{L}_\text{sketch}$ & & 0.699 & 8.747 & \textbf{0.913} & 9.795 & 0.0727 & 0.0826 & 0.0352\\
    Ours (\textbf{Point2Cyl}) & &\textbf{0.736} &\textbf{ 8.547 }& 0.911 & \textbf{8.137} & \textbf{0.0525} & \textbf{0.0704} & \textbf{0.0305}\\    
    \midrule
    H.V. + $\mathbf{N_J}$  &  & 0.395 & 13.894 & 0.594 & 58.907 & 0.1263 & 0.1662 & 0.0813  \\
    D.P. &  & - & - & - & 38.786 & 0.1392 & 2.1174 & 0.7815\\
    w/o $\mathcal{L}_\text{sketch} + \mathbf{N_J}$ & 0.01 & 0.625 & 13.894 & 0.869 & 16.544 & 0.0914 & 0.1100 & 0.0567\\
    w/o $\mathcal{L}_\text{sketch}$ &  & 0.631 & 11.901 & 0.869 & 14.542 & 0.0916 & 0.1079 & \textbf{0.0565}\\
    Ours (\textbf{Point2Cyl}) &  & \textbf{0.698} & \textbf{10.347} & \textbf{0.897} & \textbf{11.629} & \textbf{0.0760} & \textbf{0.0967} & 0.0606\\
    \midrule
    H.V. + $\mathbf{N_J}$ &  & 0.398 & 17.617 & 0.590 & 59.327 & 0.1267 & 0.1780 & 0.0902  \\
    D.P. &  & - & - & - & 45.907 & 0.1805 & 2.1277 & 1.6262\\
    w/o $\mathcal{L}_\text{sketch} + \mathbf{N_J}$ & 0.02 & 0.659 & 17.617 & 0.889 & 14.976 & 0.0820 & 0.1077 & 0.0600\\
    w/o $\mathcal{L}_\text{sketch}$ &  & 0.674 & 12.816 & 0.879 & \textbf{11.553} & 0.0796 & 0.1074 & \textbf{0.0595}\\
    Ours (\textbf{Point2Cyl}) &  & \textbf{0.679} & \textbf{12.691} & \textbf{0.891} & 11.934 & \textbf{0.0755} & \textbf{0.1056} & 0.0621\\
    \bottomrule
  \end{tabular}
    \caption{Experiment on noisy point clouds. $\sigma=0.00$ corresponds to the results reported in our main paper.}
    \label{tab_noise}
\end{table*}

\begin{table*}
    \centering
  \begin{tabular}{l|c|c|c|c|c|c|c}
    \toprule
    &Seg.$\uparrow$ & Norm.$(\degree)\downarrow$ & B.B.$\uparrow$ & E.A.$(\degree)\downarrow$ & E.C.$\downarrow$ &  Fit Cyl $\downarrow$ & Fit Glob$\downarrow$   \\
    \midrule
    w/ E.A. supervision  & 0.648 & 10.702 & 0.898 & 8.901 & 0.0735 & 0.1003 & 0.0490  \\
    w/ E.C. supervision   & 0.672 & 8.990 & 0.904 & 10.147 & 0.0733 & 0.0985 & 0.0556\\
    w/ E.A. $\&$ E.C. supervision & 0.633 & 10.508 & 0.897 & 9.166 & 0.0824 & 0.1069 & 0.0560 \\
    Ours (\textbf{Point2Cyl}) & \textbf{0.736} & \textbf{8.547} & \textbf{0.911} & \textbf{8.137} & \textbf{0.0525} & \textbf{0.0704} & \textbf{0.0305}\\
    \bottomrule
  \end{tabular}
    \caption{Experiment on adding additional direct supervision losses during training.}
    \label{tab_loss_evaluations}
\end{table*}

\begin{figure*}
    \centering
    \includegraphics[width=0.8\linewidth]{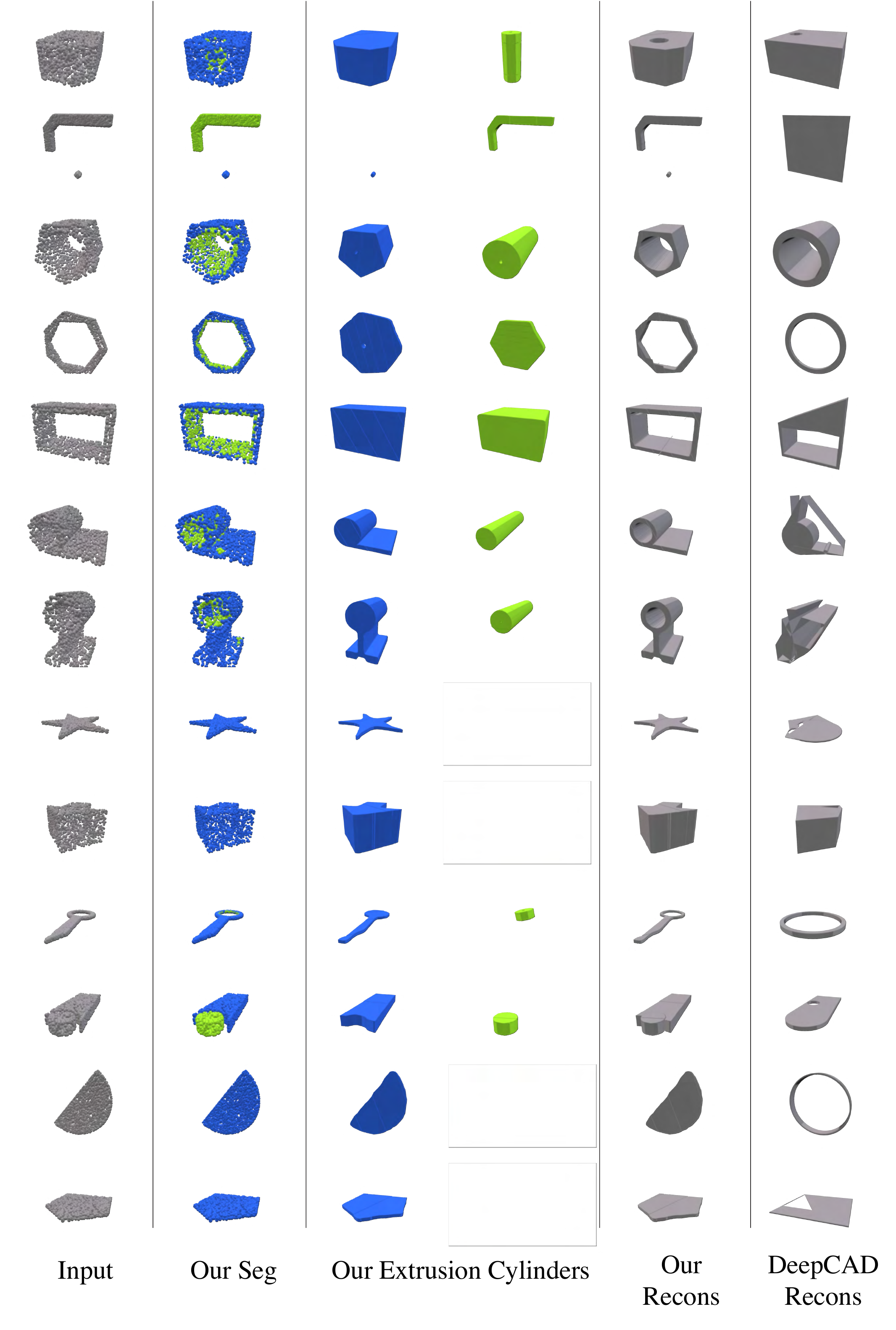}
    \caption{Additional qualitative examples from our Point2Cyl on the DeepCAD dataset. We also show comparisons with the conditional generation extension to DeepCAD~\cite{wu2021deepcad} and show that our approach result in output models that better match the input.}
    \label{fig:supp_quali}
\end{figure*}

\subsubsection{Ablation on Noisy Data}
We further experiment on adding noise to the input point clouds at both training and test time. We randomly perturb the points along the normal direction with a uniform noise between $[-\sigma, \sigma]$. Results are shown in \cref{tab_noise}, and we see that our Point2Cyl is able to tolerate noisy inputs without large performance drops. Experiments are on the Fusion Gallery dataset.

\subsubsection{Ablation on Additional Loss Functions}
We further ablate on adding additional loss functions. We experiment on adding additional loss functions that directly supervises for the extrusion axis (E.A.) and extrusion center (E.C.) based on the estimates from our parameter estimation module. Results are shown in  \cref{tab_loss_evaluations}. We see that adding these additional losses during training does not improve our performance. Experiments are on the Fusion Gallery dataset.

\subsubsection{Failure Cases and examples on Real Scans}
Fig.~\ref{fig:realscans_failurecases}-a shows some examples of failure cases. Our approach is challenged by thinly separated extrusion cylinders (1st) and thin extrusion cylinders with few barrel points (2nd) resulting in poor reconstruction. Fig.~\ref{fig:realscans_failurecases}-b shows examples of reconstructions from real scans of~\cite{spfn}.


\end{document}